\documentclass{article}

\usepackage{microtype}
\usepackage{graphicx}
\usepackage{subfigure}
\usepackage{booktabs} 

\usepackage{hyperref}

\usepackage[accepted]{icml2024}

\usepackage{amsmath}
\usepackage{amssymb}
\usepackage{mathtools}
\usepackage{amsthm}

\usepackage[capitalize,noabbrev]{cleveref}

\theoremstyle{plain}
\newtheorem{theorem}{Theorem}[section]

\newtheorem{lemma}[theorem]{Lemma}

\theoremstyle{definition}

\theoremstyle{remark}

\def\indicator{\mathbbm{1}}

\usepackage[textsize=tiny]{todonotes}

\usepackage{amsmath}
\usepackage{amsfonts,bm}
\usepackage{multirow}
\usepackage{amsthm}
\usepackage{xcolor}

\newcommand{\revise}[1]{\textcolor{black}{#1}}
\newcommand{\NOTE}[1]{\textcolor{red}{[NOTE: #1]}}
\newcommand{\yufei}[1]{\textcolor{black}{#1}}
\usepackage{graphicx}
\usepackage{wrapfig}
\usepackage{booktabs}       
\usepackage{paralist}
\usepackage{bbm}

\usepackage{amsmath,amsfonts,bm}

\def\eqref#1{equation~\ref{#1}}

\def\ceil#1{\lceil #1 \rceil}

\def\1{\bm{1}}

\DeclareMathAlphabet{\mathsfit}{\encodingdefault}{\sfdefault}{m}{sl}
\SetMathAlphabet{\mathsfit}{bold}{\encodingdefault}{\sfdefault}{bx}{n}

\def\gC{{\mathcal{C}}}

\def\gS{{\mathcal{S}}}

\newcommand{\E}{\mathbb{E}}

\newcommand{\R}{\mathbb{R}}

\usepackage{paralist}
\usepackage{url}

\usepackage{amsmath,amsthm,amssymb,multirow,paralist,mathrsfs,amsfonts,dsfont,tikz}

\def\yanblue{\textcolor{blue}}
\def\yanred{\textcolor{red}}
\def\exponential{\text{exp}}

\def\rank{\text{rank}}

\def\Beta{\text{Beta}}

\def\tr{\text{tr}}

\def\calX{\mathcal X}
\def\calY{\mathcal Y}

\def\calC{\mathcal C}

\def\calD{\mathcal D}

\def\calP{\mathcal P}

\def\calF{\mathcal F}

\def\E{\mathbb E}
\def\P{\mathbb P}
\def\R{\mathbb R}
\def\I{\mathbb I}

\def\yanred{\textcolor{red}}
\def\yanblue{\textcolor{blue}}
\newcommand{\CUT}[1]{}
\def\code#1{\texttt{#1}}

\icmltitlerunning{The Pitfalls and Promise of Conformal Inference Under Adversarial Attacks}

\begin{document}

\twocolumn[
\icmltitle{The Pitfalls and Promise of Conformal Inference Under Adversarial Attacks}



\icmlsetsymbol{equal}{*}

\begin{icmlauthorlist}
\icmlauthor{Ziquan Liu}{qmul}
\icmlauthor{Yufei Cui}{mcgill}
\icmlauthor{Yan Yan}{wsu}
\icmlauthor{Yi Xu}{dtu}
\icmlauthor{Xiangyang Ji}{tsinghua}
\icmlauthor{Xue Liu}{mcgill}
\icmlauthor{Antoni B. Chan}{cityu}
\end{icmlauthorlist}

\icmlaffiliation{qmul}{Queen Mary University of London}
\icmlaffiliation{mcgill}{McGill University, Mila}
\icmlaffiliation{wsu}{Washington State University}
\icmlaffiliation{dtu}{Dalian University of Technology}
\icmlaffiliation{tsinghua}{Tsinghua University}
\icmlaffiliation{cityu}{City University of Hong Kong}

\icmlcorrespondingauthor{Ziquan Liu}{ziquan.liu@qmul.ac.uk}

\icmlkeywords{Machine Learning, ICML}

\vskip 0.3in
]



\printAffiliationsAndNotice{}  

\begin{abstract}
In safety-critical applications such as medical imaging and autonomous driving, where decisions have profound implications for patient health and road safety, it is imperative to maintain both high adversarial robustness to protect against potential adversarial attacks and reliable uncertainty quantification in decision-making. 
With extensive research focused on enhancing adversarial robustness through various forms of adversarial training (AT), a notable knowledge gap remains concerning the uncertainty inherent in adversarially trained models. To address this gap, this study investigates the uncertainty of deep learning models by examining the performance of conformal prediction (CP) in the context of standard adversarial attacks within the adversarial defense community. It is first unveiled that existing CP methods do not produce informative prediction sets under the commonly used $l_{\infty}$-norm bounded attack if the model is not adversarially trained, which underpins the importance of adversarial training for CP. Our paper next demonstrates that the prediction set size (PSS) of CP using adversarially trained models with AT variants is often worse than using standard AT, inspiring us to research into CP-efficient AT for improved PSS. We propose to optimize a Beta-weighting loss with an entropy minimization regularizer during AT to improve CP-efficiency, where the Beta-weighting loss is shown to be an upper bound of PSS at the population level by our theoretical analysis. Moreover, our empirical study on four image classification datasets across three popular AT baselines validates the effectiveness of the proposed Uncertainty-Reducing AT (AT-UR). 

\end{abstract}

\section{Introduction}
The research into adversarial defense has been focused on improving adversarial training with various strategies, such as logit-level supervision \citep{zhang2019theoretically,cui2021learnable} and loss re-weighting \citep{wang2019improving,liu2021probabilistic}. However, the predictive uncertainty of an adversarially trained model is a crucial dimension of the model in safety-critical applications such as healthcare \citep{razzak2018deep}, and is not sufficiently understood. Existing works focus on calibration uncertainty \citep{stutz2020confidence,qin2021improving,kireev2022effectiveness}, without investigating practical uncertainty quantification of a model, e.g., prediction sets in image classification \citep{shafer2008tutorial,angelopoulos2020uncertainty,romano2020classification}. 

On the other hand, the research into conformal prediction (CP) has been extended to non-i.i.d. (identically independently distributed) settings, including distribution shifts \citep{gibbs2021adaptive} and toy adversarial noise \citep{ghosh2023probabilistically,gendler2021adversarially}. However, there is little research work on the performance of CP under standard adversarial attacks in the adversarial defense community, such as PGD-based attacks \citep{madry2018towards,croce2020reliable} with $l_{\infty}$-norm bounded perturbations. For example, \cite{gendler2021adversarially} and \cite{ghosh2023probabilistically} only consider $l_2$-norm bounded adversarial perturbations with a small attack budget, e.g., $\epsilon=0.125$ for the CIFAR dataset \citep{krizhevsky2009learning}. In contrast, the common $l_2$-norm bounded attack budget in the adversarial defense community reaches $\epsilon=0.5$ on CIFAR \citep{croce2020reliable}. In other words, existing research on adversarially robust conformal prediction is not practical enough to be used under standard adversarial attacks. 
\begin{figure*}[t]
     \centering
     \includegraphics[width=0.90\textwidth]{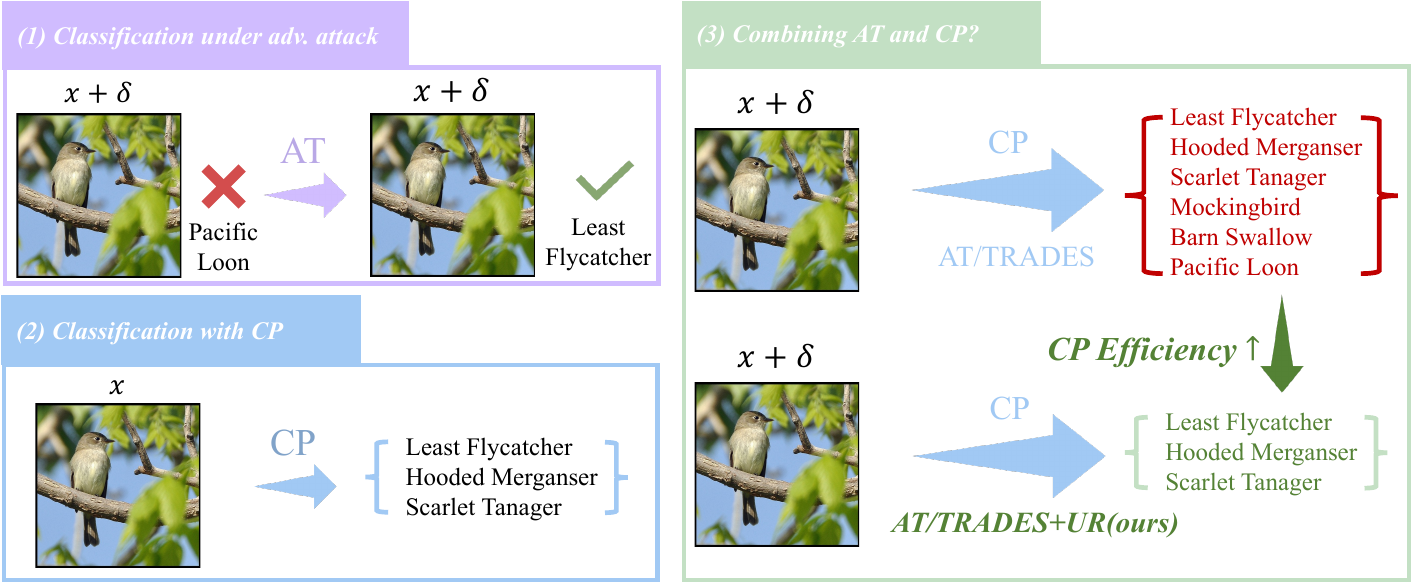}
    \caption{The proposed uncertainty-reducing adversarial training (AT-UR) improves the CP-efficiency of existing adversarial training methods like AT, FAT and TRADES. (1) AT improves the Top-1 robust accuracy of a standard model; (2) \revise{CP generates a prediction set with a pre-specified coverage guarantee for an input image, but for models not adversarially trained, CP fails to generate informative prediction sets, as the PSS is almost the same as the class number, when models being attacked (Fig.~\ref{fig:pitfalls_std_model})}; (3) When using CP in an adversarially trained model, the prediction set size is generally large, leading to inefficient CP. Our AT-UR substantially improves the CP-efficiency of existing AT methods. }
    \label{fig:main_figure}
\end{figure*}

In this context, our paper is among the first research papers to explore uncertainty of deep learning models within the framework of CP in the presence of a \emph{standard} adversary. We first present an empirical result that shows the failure of three popular CP methods on non-robust models under standard adversarial attacks, indicating the necessity of using adversarial training (AT) during the training stage. Next, we show the CP performance of three popular AT methods, finding that advanced AT methods like TRADES \citep{zhang2019theoretically} and MART \citep{wang2019improving} substantially increase the PSS in CP even though they improve the Top-1 robust accuracy. This key observation inspires us to develop the uncertainty-reducing AT (AT-UR) to learn an adversarially robust model with improved \emph{CP-efficiency} \citep{angelopoulos2020uncertainty}, meaning that CP uses a smaller PSS to satisfy the coverage. The proposed AT-UR consists of two training techniques, Beta weighting and entropy minimization, based on our observation about the two major factors that affect PSS: True Class Probability Ranking (TCPR) and prediction entropy, both defined in Sec.~\ref{sec:at_ur}. Our theoretical analysis on the Beta-weighting loss reveals that the proposed weighted loss is an upper bound for the PSS at the population level. The proposed AT-UR is demonstrated to be effective at reducing the PSS of models on multiple image classification datasets. In summary, there are four major contributions of this paper.

\begin{compactenum}
\item We test several CP methods under commonly used adversarial attacks in the adversarial defense community. \revise{It turns out that for models not adversarially trained, CP cannot generate informative prediction sets.} Thus, adversarial training is necessary for CP to work under adversarial attacks.
\item We test the performance of adversarially trained models with CP and demonstrate that improved AT often learns a more uncertain model and leads to less efficient CP with increased PSS.
\item We propose uncertainty-reducing AT (AT-UR) to learn a CP-efficient and adversarially robust model by minimizing the entropy of predictive distributions and a weighted loss where the weight is a Beta density function of TCPR.
\item Our main theorem shows that at the population-level, the Beta-weighting loss is an upper bound for the targeted PSS, so minimizing the weighted loss leads to reduced PSS in theory. This theoretical result corroborates our hypothesis that optimizing the \emph{promising} samples with high weights leads to reduced PSS.
\item Our empirical study demonstrates that the proposed AT-UR learns adversarially robust models with substantially improved CP-efficiency on four image classification datasets across three AT methods, validating our major theoretical result. 
\end{compactenum}

The paper is structured as follows. Section 2 discusses related works and Section 3 introduces mathematical notations and two key concepts in this paper. Section 4 shows the pitfalls of three CP methods under standard attacks when the model is not robustly trained and the low CP-efficiency of two improved AT methods and motivates us to develop the AT-UR introduced in Section 5. Our major empirical results are shown in Section 6 and we conclude the paper in Section 7. Our code is available at \textcolor{magenta}{\url{https://github.com/ziquanliu/ICML2024-AT-UR}}.

\CUT{
1. Entropy-based focal loss (sometimes working well)

2. CP distillation (not working well)

3. label smoothing/regularization (not working well)

4. Beta-distribution rank focal (best)

5. hard sample abstain (promising)

6. small-loss sample drop (reduce overfitting)
}

\section{Related Works}
\textbf{Adversarial Robustness.} The most effective approach to defending against adversarial attacks is adversarial training (AT) \citep{madry2018towards}. There is a sequence of works following the vanilla version of AT based on projected gradient descent (PGD), including regularization \citep{qin2019adversarial,liu2022boosting,liu2021improve}, logit-level supervision \citep{zhang2019theoretically,cui2021learnable} and loss re-weighting \citep{wang2019improving,liu2021probabilistic}. Existing methods on regularization focus on improving Top-1 robust accuracy by training the model with certain properties like linearization \citep{qin2019adversarial} and large margins \citep{liu2022boosting}. In contrast, our work focuses on the PSS, i.e., the efficiency of CP, in adversarially trained models by regularizing the model to have low prediction entropy.  The entropy minimization regularization also entails logit-level supervision as in \cite{zhang2019theoretically}. In comparison, our proposed approach, AT-EM, enhances CP efficiency, whereas TRADES (Zhang et al., 2019) impedes CP-efficiency. \revise{The most related work is \cite{gendler2021adversarially} which also studies CP under adversarial attacks. However, there are two fundamental differences: 1) \cite{gendler2021adversarially} only considers a small attack budget under $l_2$-norm bounded attacks, while our work investigates CP under common adversarial attacks in adversarial defense literature with $l_{\infty}$-norm bounded attacks; 2) Our paper shows that AT is essential for CP to work under strong adversarial attacks and proposes novel AT methods to learn a CP-efficient and adversarially-robust model, while \cite{gendler2021adversarially} only considers the post-training stage. Our experiment validates that  \cite{gendler2021adversarially} fails when there are strong adversarial attacks (Fig.~\ref{fig:pitfalls_std_model}). }

\textbf{Uncertainty Quantification.} \yufei{Uncertainty quantification aims to provide an uncertainty measure for a machine learning system's decisions. Within this domain, Bayesian methods stand out as a principled approach, treating model parameters as random variables with distinct probability distributions. This is exemplified in Bayesian Neural Networks (BNNs), which place priors on network weights and biases, updating these with posterior distributions as data is observed~\citep{gal2016dropout,kendall2017uncertainties,cui2020accelerating}. However, the large scale of modern neural networks introduces challenges for Bayesian methods, making prior and posterior selection, and approximate inference daunting tasks~\citep{kingma2015variational,cui2021bayesian,cui2023variational,cui2023bayesmil}. This can sometimes compromise the optimal uncertainty quantification in BNNs. In contrast, the frequentist approach offers a more direct route to uncertainty estimation. It views model parameters as fixed yet unknown, deriving uncertainty through methods like conformal prediction~\citep{vovk1999machine,ghosh2023probabilistically,gendler2021adversarially}. While Bayesian methods integrate prior beliefs with data, their computational demands in large networks can be overwhelming, positioning the straightforward frequentist methods as a viable alternative for efficient uncertainty quantification.} Thus, our paper investigates the uncertainty of adversarially trained models via CP. Note that our work is fundamentally different from existing research on uncertainty calibration for AT \citep{stutz2020confidence,qin2021improving,kireev2022effectiveness}, as our focus is to produce a valid prediction set while uncertainty calibration aims to align accuracy and uncertainty. Finally, \cite{einbinder2022training} proposes to train a model with uniform conformity scores on a calibration set in standard training, while our work proposes CP-aware adversarial training to reduce PSS.

\section{Preliminary}
Before diving into the details of our analysis and the proposed method, we first introduce our mathematical notations, adversarial training and conformal prediction.

\textbf{Notations.} Denote a training set with $m$ samples by $\calD_\tr = \{(x_i, y_i)\}_{i=1}^{m}$.
Suppose each data sample $(x_i, y_i) \in \calX \times \calY$ is drawn from an underlying distribution $\calP$ defined on the space $\calX \times \calY$, where $x_i$ and $y_i$ are the feature and label, respectively.
Particularly, we consider the classification problem and assume that there are $K$ classes, i.e., $\calY = \{1,..., K\}$ (we denote $[K] = \{1, ..., K\}$ for simplicity).
Let $f_{\theta} : \calX \rightarrow \Delta_p^K$ denote a predictive model from a hypothesis class $\calF$ that generates a $K$-dimensional probability simplex: $\Delta_p^K = \{ v \in [0, 1]^K : \sum_{k=1}^K v_k = 1 \}$. $\theta$ is the model parameter we optimize during training. A loss function $\ell : \calY \times \calY \rightarrow \R$ is used to measure the difference between the prediction made by $f_{\theta}(x)$ and the ground-truth label $y$. 

To measure the performance of $f_{\theta}$ in the sense of population over $\calP$, the {\it true risk} is typically defined as $R(f_{\theta}) = \P_{(x, y) \sim \calP} [ f_{\theta}(x) \neq y ]$.
Unfortunately, $R(f)$ cannot be realized in practice, since the underlying $\calP$ is unreachable.
Instead, the {\it empirical risk} $\widehat R(f_{\theta}) = \frac{1}{m} \sum_{i=1}^m \I[ f_{\theta}(x_i) \neq y_i ]$ is usually used to estimate $R(f_{\theta})$, where $\I[\cdot]$ is the indicator function.
The estimation error of $\widehat R(f_{\theta})$ to $R(f_{\theta})$ is usually referred to as generalization error bound and can be bounded by a standard rate $O(1/\sqrt{m})$.
To enable the minimization of empirical risk, a loss function $\ell$ is used as the surrogate of $\I[\cdot]$, leading to the classical learning paradigm empirical risk minimization (ERM): $\min_{f_{\theta} \in \calF} \widehat L(f_{\theta}) = \frac{1}{m} \sum_{i=1}^m \ell(f_{\theta}(x_i), y_i)$. In this work, we use the standard cross-entropy loss as the loss function where $j$ is the index for a $j$th element in a vector,
\setlength{\belowdisplayskip}{1.0pt} \setlength{\belowdisplayshortskip}{1.0pt}
\setlength{\abovedisplayskip}{1.0pt} \setlength{\abovedisplayshortskip}{1.0pt}
\begin{align}
    \ell(f_{\theta}(x_i), y_i)=-\sum_{j=1}^Ky_{ij}\log(f_{\theta}(x_i)_j).
\end{align}

\textbf{Adversarial training.} Write the loss for sample ($x_i,y_i$) in adversarial training as $l(\tilde x_i,y_i)$, where $\tilde x_i=x_i+\delta_i$ and $\delta_i$ is generated from an adversarial attack, e.g., PGD attack \citep{madry2018towards}. The vanilla adversarial training minimizes the loss with uniform weights for a mini-batch with $B$ samples, i.e.,
\begin{align}
    \nabla f_{\theta} = \nabla \frac{1}{B}\sum_{i=1}^Bl(f_{\theta}(\tilde x_i)_j,y_i),
\end{align}
where $\nabla f_{\theta}$ is the gradient in this mini-batch step optimization with respect to $\theta$.

\begin{figure*}[t]
     \centering
     \includegraphics[width=\textwidth]{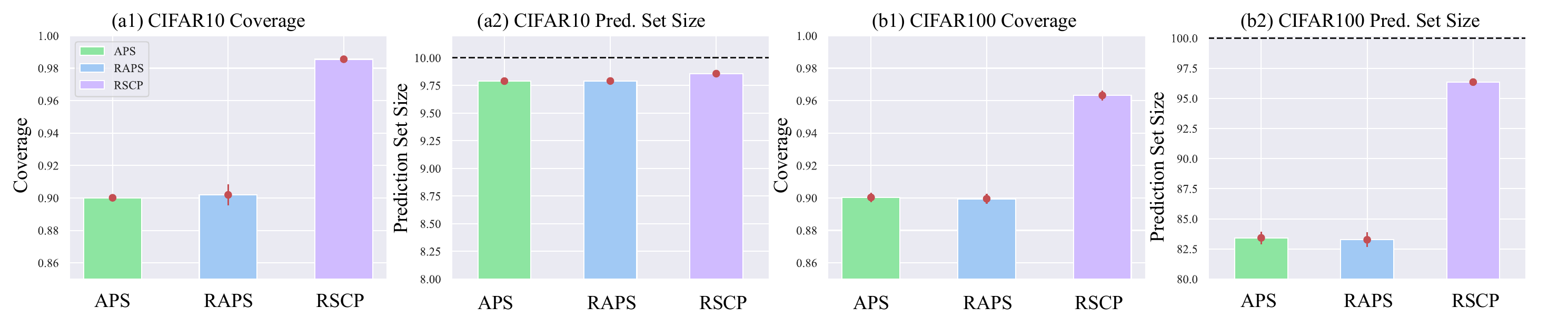}
    \caption{The performance of three representative CP methods using non-robust models under standard adversarial attacks in the adversarial defense community. The red line denotes means standard deviation of the metric. For comparison, the average PSS for normal images is 1.03 and 2.39 for CIFAR10 and CIFAR 100. \revise{See Sec.~\ref{sec:experiment:setting} for details of the experiment.}}
    \label{fig:pitfalls_std_model}
\end{figure*}

\begin{table}
\centering
\resizebox{1.0\linewidth}{!}{
\begin{tabular}{|c|c|c|c|c|c|}
\hline
&Dataset & CIFAR10 & CIFAR100 & Caltech256 & CUB200 \\
\hline
\multirow{4}{*}{AT} & Rob. Coverage &93.25(0.45) & 91.99 (0.61) & 94.35(0.81) & 91.87(0.90) \\
& Rob. Set Size & 2.54(0.04) & 14.29(0.59) & 23.73(1.68) & \textbf{17.75(0.71)}\\
& Clean Acc. &89.76(0.15) & 68.92(0.38) & 75.28(0.51) & 65.36(0.27)\\
& Rob. Acc. & 47.75(0.94) & 24.89(1.05) & 43.41(0.59) & 23.34(0.41) \\
\hline
\multirow{4}{*}{TRADES} & Rob. Coverage & 93.01(0.46) & 91.75(0.72)	 &  94.33(0.49) & 92.03(0.59)	 \\
& Rob. Set Size & \textbf{2.49(0.03)} & \textbf{12.22(0.50)} & \textbf{22.82(1.37)} & 22.29(0.96) \\
& Clean Acc. &87.31(0.27) & 62.83(0.33) & 69.57(0.25) & 58.16(0.38) \\
& Rob. Acc. & \textbf{50.52(0.31)} & 26.31(0.29) & 43.69(0.45) & 22.30(0.21)
\\
\hline
\multirow{4}{*}{MART} & Rob. Coverage & 94.56(0.36) &92.24(0.73) &  95.28(0.66) & 92.24(0.70)\\
& Rob. Set Size &3.07(0.05) & 15.26(0.65) &28.89(2.83) & 25.01(1.23) \\
& Clean Acc. &85.43(0.24) & 59.66(0.26) & 69.68(0.31) & 58.72(0.18)\\
& Rob. Acc. & 48.66(0.38) & \textbf{27.41(0.55)} & \textbf{44.61(0.37)}  & \textbf{24.01(0.43)}\\
\hline
\end{tabular}
}
\caption{CP and Top-1 accuracy of three popular adversarial defense methods under AutoAttack \cite{croce2020reliable}. Bold numbers are the best PSS and robust accuracy. }
\label{tab:popular_AT}
\end{table}

\textbf{Conformal prediction (CP).} CP is a distribution-free uncertainty quantification method and can be used in a wide range of tasks including both regression and classification \citep{vovk1999machine,vovk2005algorithmic}. This paper focuses on the image classification task, where CP outputs a prediction set instead of the Top-1 predicted class as in a standard image classification model, and satisfies a coverage guarantee. Mathematically, CP maps an input sample $x$ to a prediction set $\gC(x)$, which is subset of $[K]=\{1,\cdots,K\}$, with the following coverage guarantee,
\begin{align}
P(y\in\gC(x))\geq 1-\alpha,
\end{align}
where $1-\alpha$ is a pre-defined confidence level such as 90\%, meaning that the prediction set will contain the ground-truth label with 90\% confidence for future data. This paper mainly considers the \emph{split conformal prediction}, an efficient CP approach applicable to any pre-trained black-box classifier \citep{papadopoulos2002inductive,lei2018distribution} as it does not need to re-train the classifier with different train-calibration-test splits. 

The prediction set of CP is produced by the calibration-then-test procedure. In the context of a classification task, we define a prediction set function $\gS(x,u;\pi,\tau)$, where $u$ is a random variable sampled from a uniform distribution $\text{Uniform}[0,1]$ independent of all other variables, $\pi$ is shorthand for the predictive distribution $f_{\theta}(x)$, and $\tau$ is a threshold parameter that controls the size of the prediction set. An increase in the value of $\tau$ leads to an expansion in the size of the prediction set within $\mathcal{S}(x,u;\pi,\tau)$. We give one example \citep{romano2020classification} of the function $\gS$ in Appendix~\ref{app_sec:aps}. The calibration process computes the smallest threshold parameter $\hat\tau_{\text{cal}}$ to achieve an empirical coverage of $(1-\alpha)(n_c+1)/n_c$ on the calibrations set with $n_c$ samples. For a test sample $x^*$, a prediction set is the output of the function $\mathcal{S}(x^*,u;\pi^*,\hat\tau_{\text{cal}})$. 

\section{Necessitate AT for Robust and Efficient Coverage.}
\noindent\textbf{The pitfalls of CP under strong adversarial attacks}. We test the performance of three conformal prediction methods, i.e., APS (Adaptive Prediction Sets) \citep{romano2020classification}, RAPS (Regularized Adaptive Prediction Sets) \citep{angelopoulos2020uncertainty} and RSCP (Randomly Smoothed Conformal Prediction) \citep{gendler2021adversarially}, under standard adversarial attacks. Specifically, for APS and RAPS, we use PGD100 adversarial attacks with $l_{\infty}$-norm bound and attack budget $\epsilon=8/255=0.0314$. For RSCP, we adopt PGD20 with an $l_2$-norm bound, in accordance with the original paper's settings, but with a larger attack budget of $\epsilon=0.5$ as in RobustBench \citep{croce2020reliable}. If not specified otherwise, we use adversarial attack PGD100 with $l_{\infty}$ norm and $\epsilon=8/255=0.0314$ to generate adversarial examples throughout this paper. 

Fig.~\ref{fig:pitfalls_std_model} shows the coverage and PSS of three CP methods on CIFAR10 and CIFAR100 when models are trained in a standard way, i.e., without adversarial training. Although all CP methods have good coverages, their prediction set sizes are close to the number of classes in both datasets as the classifier is completely broken under strong adversarial attacks. In contrast, when the same models are applied to standard images, the PSS are 1.03 and 2.39 for CIFAR10/CIFAR100. This result reveals that adversarial training is indispensable if one wants to use CP to get reasonable uncertainty quantification for their model in an adversarial environment. Therefore, in next section, we test AT and two improved AT methods to investigate the performance of CP for adversarially trained models.

\CUT{
\begin{table}[t]
\centering
\begin{tabular}{|l|l|l|l|l|}
\hline
             & CIFAR10 & CIFAR100 & Caltech256 & CUB200 \\ \hline
Adv Coverage &    90.01(0.03)     &   90.03(0.28)       &            &        \\
Adv Set Size &     9.79 (0.01)    &    83.42(0.53)      &            &        \\ \hline
\end{tabular}
\caption{APS}
\end{table}
}

\CUT{
\begin{table}[t]
\centering
\begin{tabular}{|l|l|l|l|l|}
\hline
             & CIFAR10 & CIFAR100 & Caltech256 & CUB200 \\ \hline
Adv Coverage &   90.190 (0.654)      &    89.948 (0.312)      &            &        \\
Adv Set Size &   9.791 (0.022)      &    83.261 (0.608)      &            &        \\ \hline
\end{tabular}
\caption{RAPS}
\end{table}
}

\CUT{
\begin{table}[t]
\centering
\begin{tabular}{|l|l|l|l|l|}
\hline
             & CIFAR10 & CIFAR100 & Caltech256 & CUB200 \\ \hline
Adv Coverage &   98.5528(0.1397)      & 96.3204(0.3021)         &            &        \\
Adv Set Size &   9.857516(0.007097)      &     96.353580(0.129677)     &            &        \\ \hline
\end{tabular}
\caption{RSCP}
\end{table}
}

\noindent\textbf{Improved AT Compromises Conformal Prediction's Efficiency}. We test three popular adversarial training methods, i.e., AT \citep{madry2018towards}, TRADES \citep{zhang2019theoretically} and MART \citep{wang2019improving}, using APS as the conformal prediction method under a commonly used adversarial attack, AutoAttack with $l_{\infty}$-norm and $\epsilon=8/255=0.0314$. See more detailed experimental settings in Sec.~\ref{sec:experiment}. Tab.~\ref{tab:popular_AT} shows their coverage and PSS, as well as clean and robust accuracy on four datasets. The results demonstrate that while the two enhanced adversarial training methods, TRADES and MART, effectively improve the Top-1 accuracy in the presence of adversarial attacks, they often lead to an increase in the size of the prediction set, consequently yielding a less CP-efficient model. In other words, the improvement in Top-1 accuracy does not necessarily lead to less uncertainty. Therefore, to design a new AT method that learns an adversarially robust model with efficient CP, a deep investigation into the PSS is necessary. In the following section, we identify two major factors that play an important role in controlling the PSS through our empirical study.

\section{Uncertainty-Reducing Adversarial Training}
\label{sec:at_ur}
This section investigates two factors highly correlated with PSS and introduces our uncertainty-reducing adversarial training method.

\begin{figure}[t]
     \centering
     \includegraphics[width=\linewidth]{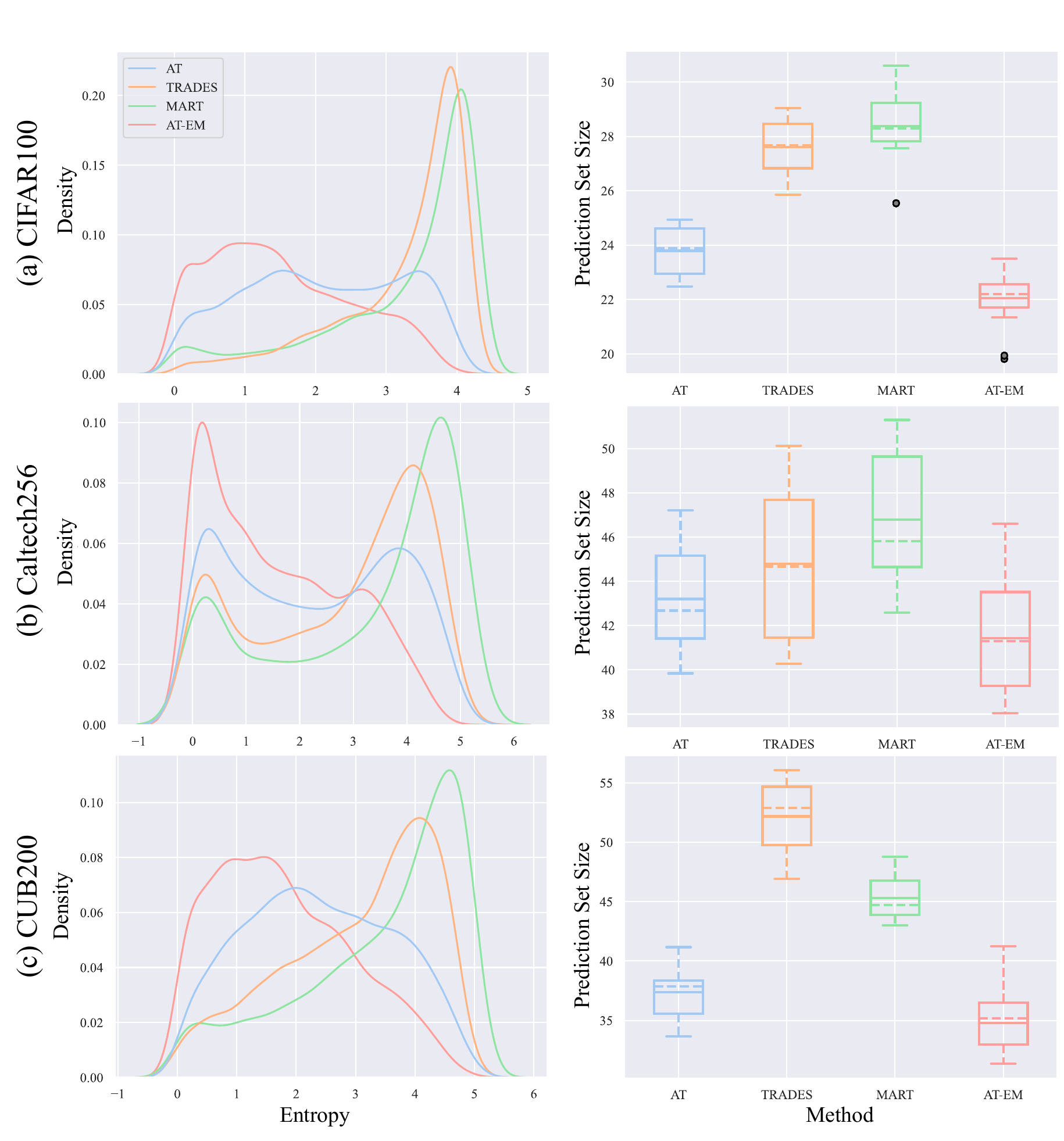}
    \caption{\textbf{(Left)}: The kernel density estimation for predictive distribution's entropy on adversarial test sets. \textbf{(Right)}: Box plot of PSS of three AT baselines and AT-EM. AT-EM effectively controls prediction entropy and improves CP-efficiency. See Tab.~\ref{tab:popular_AT} and Tab.~\ref{tab:main} for their coverages.} 
    \label{fig:all_entropy}
\end{figure}

\subsection{Entropy Minimization for CP-Efficiency}
The PSS is closely related to the entropy of prediction distribution, as both quantities reflect the prediction uncertainty of a model. A more uniform categorical distribution has higher uncertainty, which is reflected in its higher entropy. Fig.~\ref{fig:all_entropy} visualizes the kernel density estimation (KDE) \citep{rosenblatt1956remarks,parzen1962estimation} of entropy values calculated with adversarial test samples on three datasets. It is evident that TRADES and MART learn models with predictive distributions that have higher entropy values than AT, thus increasing the PSS comparatively. 

To decrease the PSS of AT, we add an entropy minimization term to the loss function, 
\begin{align}
    \small\ell_{\text{EM}}(f_{\theta}(x_i), y_i)=-\sum_j^Ky_{ij}\log(f_{\theta}(x_i)_j)+\lambda_{\text{EM}} H(f_{\theta}(x_i)),
    \label{eqn:at_em}
\end{align}
where the regularization is the entropy function $H(f_{\theta}(x_i))=-\sum_j^Kf_{\theta}(x_i)_j\log(f_{\theta}(x_i)_j)$. We set $\lambda_{\text{EM}}$=0.3 in all of our experiments based on a hyperparameter search experiment on CIFAR100, where $\lambda_{\text{EM}}\in\{0.1,0.3,1.0,3.0\}$. The AT scheme with entropy minimization (EM) is denoted as AT-EM. This entropy term is the same as the entropy minimization in semi-supervised learning \citep{grandvalet2004semi}. However, note that our work is the first to use entropy minimization in adversarial training for improving CP-efficiency. Fig.~\ref{fig:all_entropy} also shows the KDE of entropy values on adversarial test sets using AT-EM. The reduction in predictive entropy effectively leads to a substantial decrease in the PSS of AT-EM.

The second factor that affects PSS is the distribution of True Class Probability Ranking (TCPR) on the test dataset. The TCPR is defined as the ranking of a sample $x$'s ground-truth class probability among the whole predictive probability. In equation, we sort $\pi$ with the descending order into $\hat\pi$,
\begin{align}
    \hat\pi = \{\pi_{(1)},\cdots,\pi_{(K)}\},
\end{align}
where $\pi_{(j)}\geq\pi_{(j+1)},\forall j=1,\cdots,K-1$, and $(j)$ is the sorted index. TCPR is the index $j$ in $\hat\pi$ corresponding to the ground-truth label $y$, i.e., $Sort(y)=j$. 

The TCPR matters to the PSS as we observe that a model with higher robust accuracy does not necessarily have a smaller PSS as shown in Tab.~\ref{tab:popular_AT}. This discovery indicates that improving Top-1 accuracy, i.e., the percentage of samples with TCPR=1, is not enough to learn a CP-efficient model. In particular, the model capacity might be not strong enough to fit all the adversarial training data or achieve 100\% adversarial training accuracy as a result of a strong adversary and high task complexity, e.g., a large number of classes. For instance, on CIFAR100, the robust accuracy on training data of a pre-trained ResNet50 is only around 45\% after 60 epochs of fine-tuning. 

\CUT{
\begin{figure*}[t]
     \centering
     \includegraphics[width=\textwidth]{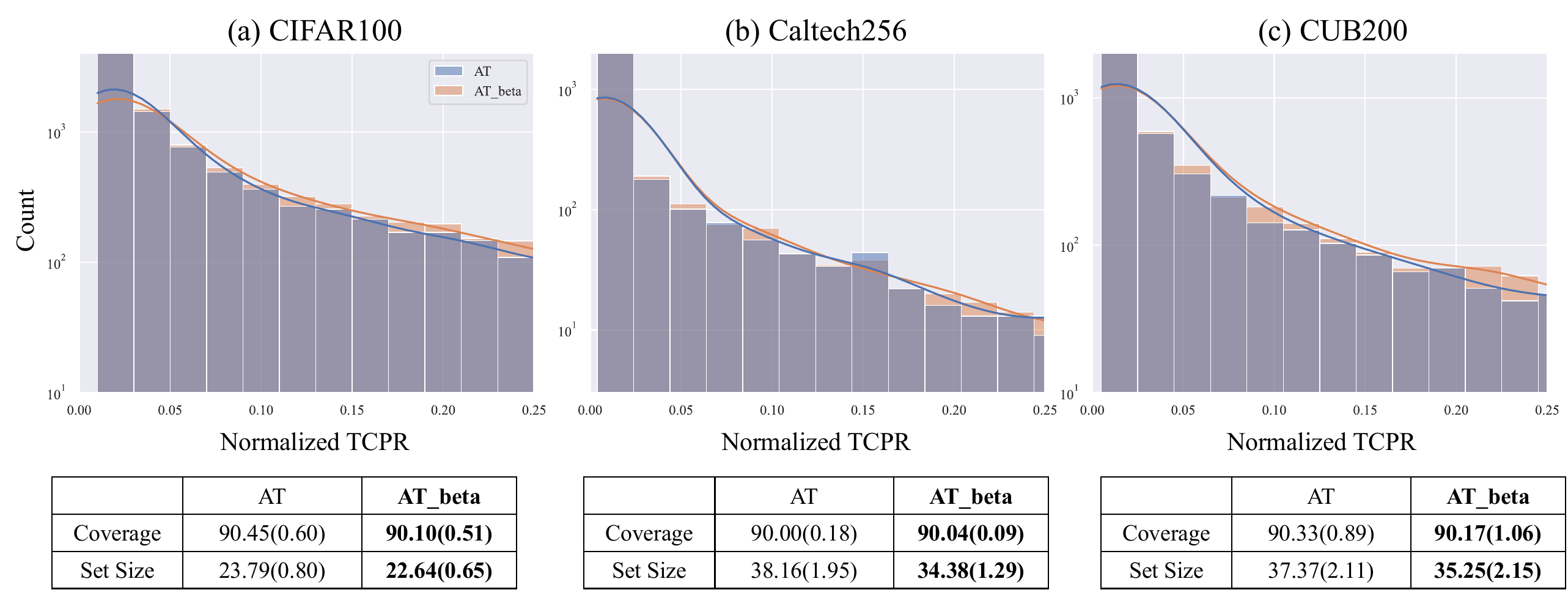}
    \caption{\textbf{Top}: The histogram and kernel density estimation of normalized TCPR on adversarial test sets. \textbf{Bottom}: The coverage and PSS of AT and AT-Beta. AT-Beta pushes the TCPR distribution towards the promising region and improves CP-efficiency.}
    \label{fig:TCPR_hist}
\end{figure*}
}

Motivated by this observation, we propose to use a Beta distribution density function (Fig.~\ref{fig:beta_function}) to weight the loss samples so that the TCPR distribution shifts towards the lower TCPR region. This design embodies our intuition that the training should focus on samples with \emph{promising} TCPR's, whose TCPR's are neither 1 nor too large, because TCPR=1 means the sample is correctly classified and a large TCPR means the sample is an outlier and probably hopeless to learn. Those samples with promising TCPR's are important to control PSS as they are the \emph{majority} of the dataset and thus largely affect the averaged PSS, see Fig.~\ref{fig:tcpr_percentage} for the percentage of promising samples throughout AT training on CIFAR100.

With the previous intuition, we propose an importance weighting scheme based on the Beta distribution density function of TCPR to learn a CP-efficient model. Let the TCPR of sample $\tilde x_i$ be $r_i\in[K]$ and the normalized TCPR be $\hat r_i\in(0,1]$. Note that in our implementation we use the index starting from 0 instead of 1, so $\hat r_i\in[0,1)$ in practice. We use the Beta distribution density function, e.g., Fig.~\ref{fig:beta_function}, to give an importance weight to sample $\tilde x_i$. We use the Beta distribution density up-shifted by 1
\begin{align}
\small
\tilde p_\Beta(z; a, b) &= 1+p_\Beta(z; a, b)\nonumber\\
\small&=1+\frac{ \Gamma(a+b) }{ \Gamma(a) \Gamma(b) } \cdot (z)^{a-1} \cdot (1-z)^{b-1},
\label{eqn:beta}
\end{align}
where $\Gamma(a)$ is the Gamma function. We use the add-1 Beta function $\tilde p_{\text{Beta}}$ for stable optimization and better performance based on our pilot study. To enforce the model to focus on samples with promising TCPR's, we use the Beta distribution with $a=1.1$ and $b\in\{3.0,4.0,5.0\}$. When $a=1.1$ and $b=5.0$, we have the Beta weighting function shown in Fig.~\ref{fig:beta_function}. The objective function of Beta-weighting AT is
\begin{align}
    \ell_{\text{Beta}}(f_{\theta}(x_i), y_i)=-\tilde 
 p_\Beta(\hat r_i; a, b)\cdot \sum_j^Ky_{ij}\log(f_{\theta}(x_i)_j)
\end{align}
We name this Beta distribution based importance weighting scheme in AT as AT-Beta. 

In summary, the proposed AT-UR consists of two methods, AT-Beta 
 and AT-EM. It also contains the combination of the two methods, i.e., 
 \small
\begin{align}
    \ell_{\text{Beta-EM}}(f_{\theta}(x_i), y_i)&=-\tilde p_\Beta(\hat r_i; a, b)\cdot \sum_j^Ky_{ij}\log(f_{\theta}(x_i)_j)\nonumber\\
    &+\lambda_{\text{EM}} H(f_{\theta}(x_i)),
\end{align}
\normalsize
denoted as AT-Beta-EM. We test the three variants of AT-UR in our experiment and observe that different image classification tasks need different versions of AT-UR. 
\subsection{Theoretical Analysis on Beta Weighting}
The previous subsection introduces the intuition behind the proposed AT-UR. This section gives the theoretical analysis on the Beta weighting, which shows a theoretical connection between Beta weighting and the PSS. We drop the subscript $\theta$ for $f_{\theta}$ to lighten the notation. Note that we leave the full proof in Appendix~\ref{app_sec:proof_thm1}.

\noindent {\bf Importance Weighting (IW) Algorithm.} 
IW assigns importance weight $\omega(x, y)$ to each sample $(x, y) \in \calD_\tr$ such that $\omega(x, y)$ is directly determined by TCPR $\hat r$.
Analogous to the empirical risk $\widehat R(f)$, we define the {\it IW empirical risk} with weights $\omega(x, y)$ for $f$ as follows
\begin{align}\label{eq:empirical_IW_risk}
\widehat R_\omega (f)
=
\frac{1}{m} \sum_{i=1}^m \omega(x_i, y_i) \cdot \ell(f(x_i), y_i) .
\end{align}
It is worth noting that restricting $\omega(x_i, y_i) = 1$ as a special case for all data samples reduces $\widehat R_\omega(f)$ to $\widehat R(f)$.

\begin{figure}
    \begin{center} 
    \includegraphics[width=0.43\textwidth]{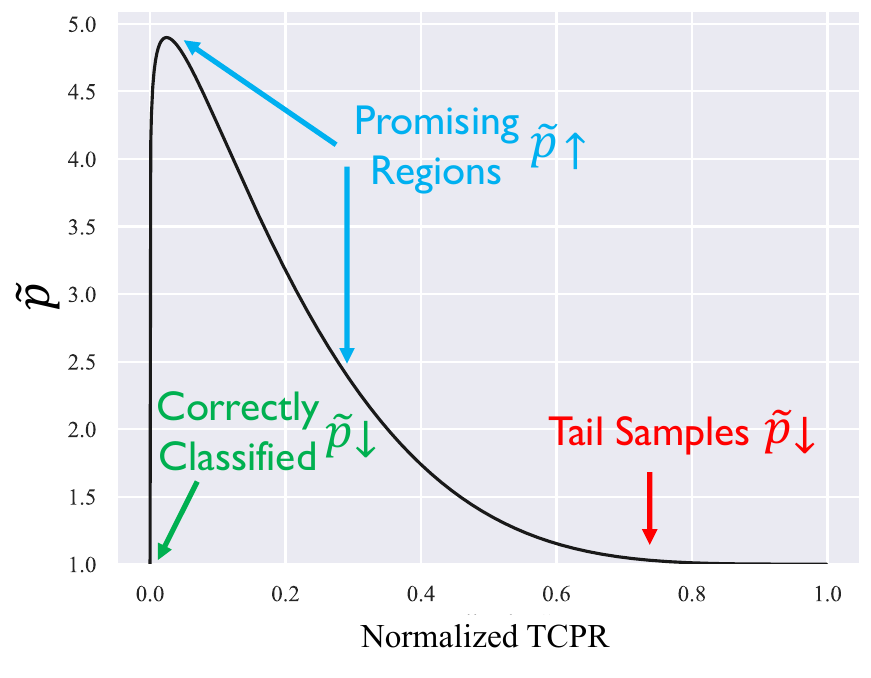}
    \caption{The Beta distribution density function $\tilde p_{\text{Beta}}$ used in our experiment. This weighting scheme increases the importance of samples in the promising region. }
    \label{fig:beta_function}
    \end{center}
\end{figure}
\subsection{Beta Weighting for CP-Efficiency}
\label{sec:tcpr}




\begin{table*}[t]
\centering
\resizebox{\linewidth}{!}{
\begin{tabular}{|c|cc|cc|cc|cc|}
\hline
     Dataset   & \multicolumn{2}{c|}{CIFAR10}     & \multicolumn{2}{c|}{CIFAR100}    & \multicolumn{2}{c|}{Caltech256}  & \multicolumn{2}{c|}{CUB200}      \\ \hline 
     Metric   & \multicolumn{1}{c|}{Cvg} & PSS & \multicolumn{1}{c|}{Cvg} & PSS & \multicolumn{1}{c|}{Cvg} & PSS & \multicolumn{1}{c|}{Cvg} & PSS \\ \hline
AT      & \multicolumn{1}{c|}{93.25(0.45)}     &   2.54(0.04)   & \multicolumn{1}{c|}{91.99 (0.61) }     &  14.29(0.59)    & \multicolumn{1}{c|}{94.35(0.81) }     &  23.73(1.68)   & \multicolumn{1}{c|}{91.87(0.90) }     &   17.75(0.71)  \\ 
AT-EM$^*$   & \multicolumn{1}{c|}{92.36(0.53) }     &  \textbf{2.45(0.04)}    & \multicolumn{1}{c|}{91.87(0.61) }     & 13.29(0.49)     & \multicolumn{1}{c|}{93.41(0.58) }     &  21.19(1.46)    & \multicolumn{1}{c|}{91.26(0.57) }     &   \textbf{16.47(0.61)}  \\ 
AT-Beta$^*$ & \multicolumn{1}{c|}{91.96(0.39)  }     &   2.50(0.04)   & \multicolumn{1}{c|}{ 91.24(0.69)  }     &  \textbf{11.61(0.40)}    & \multicolumn{1}{c|}{93.52(0.73) }     &   \textbf{18.54(1.32)}   & \multicolumn{1}{c|}{91.37(0.75)}     &   16.56(0.76)  \\ 
AT-Beta-EM$^*$   & \multicolumn{1}{c|}{92.06(0.44)}     &  2.50(0.04)    & \multicolumn{1}{c|}{ 91.13(0.63) }     & 11.78(0.47)    & \multicolumn{1}{c|}{93.50(0.69)}     &   18.56(1.32)   & \multicolumn{1}{c|}{91.93(0.68)}     & 16.67(0.58)   \\ 
\hline
FAT      & \multicolumn{1}{c|}{93.01(0.53)}     &   2.55(0.05)   & \multicolumn{1}{c|}{92.04(0.60)  }     &  13.60(0.37)    & \multicolumn{1}{c|}{93.88(0.45) }     &  22.85(0.97)   & \multicolumn{1}{c|}{91.37(0.80) }     &   17.21(0.81) \\ 
FAT-EM$^*$   & \multicolumn{1}{c|}{92.71(0.66)  }     &  2.49(0.05)    & \multicolumn{1}{c|}{91.63(0.91) }     & 14.43(2.12)    & \multicolumn{1}{c|}{93.52(0.67)  }     &  21.81(1.28)   & \multicolumn{1}{c|}{91.58(1.25) }     &   16.60(1.05)  \\ 
FAT-Beta$^*$ & \multicolumn{1}{c|}{92.08(0.54) }     &   2.55(0.04)   & \multicolumn{1}{c|}{ 90.82(0.49)  }     &  11.10(0.25)    & \multicolumn{1}{c|}{93.49(0.82) }     &   \textbf{18.07(1.30)}   & \multicolumn{1}{c|}{91.10(0.61)}     &   \textbf{16.14(0.54)}  \\ 
FAT-Beta-EM$^*$   & \multicolumn{1}{c|}{92.28(0.30) }     &  \textbf{2.43(0.02)}    & \multicolumn{1}{c|}{ 91.18(0.55) }     & \textbf{11.09(0.25)}    & \multicolumn{1}{c|}{93.55(0.56) }     &   18.16(1.04)  & \multicolumn{1}{c|}{91.40(0.62) }     & 16.39(0.49)  \\ 
\hline
TRADES      & \multicolumn{1}{c|}{93.01(0.46) }     &   2.49(0.03)   & \multicolumn{1}{c|}{91.75(0.72)  }     &  12.22(0.50)   & \multicolumn{1}{c|}{94.33(0.49)  }     &  22.82(1.37)   & \multicolumn{1}{c|}{92.03(0.59) }     &   22.29(0.96) \\ 
TRADES-EM$^*$   & \multicolumn{1}{c|}{92.22(0.19)  }     &  \textbf{2.33(0.01)}    & \multicolumn{1}{c|}{91.43(0.77) }     & 12.39(0.54)   & \multicolumn{1}{c|}{93.49(0.65) }     &   \textbf{15.03(1.32)}   & \multicolumn{1}{c|}{91.52(0.82)}     &   \textbf{17.80(2.37)}  \\ 
TRADES-Beta$^*$ & \multicolumn{1}{c|}{92.43(0.47) }     &   2.44(0.04)  & \multicolumn{1}{c|}{ 91.20(0.72) }     &  \textbf{10.76(0.44)}    & \multicolumn{1}{c|}{93.55(0.48)  }     &  17.02(1.04)   & \multicolumn{1}{c|}{90.91(0.89)}     &   18.19(0.99) \\ 
TRADES-Beta-EM$^*$   & \multicolumn{1}{c|}{92.12(0.36)  }     &  2.42(0.02)  & \multicolumn{1}{c|}{  91.22(0.62) }     & 11.09(0.38)    & \multicolumn{1}{c|}{93.13(0.33) }     &   17.27(0.72)  & \multicolumn{1}{c|}{90.96(0.93)}     & 18.28(1.22)   \\ 
\hline
\end{tabular}
}
\caption{Comparison of AT baselines and the proposed AT-UR variants denoted with $^*$, under the AutoAttack \cite{croce2020reliable}. The average coverage (Cvg) and Prediction Set Size (PSS) are presented, along with the standard deviation in parentheses. The most CP-efficient method is highlighted in bold. The result of using PGD100 attacks is in Tab.~\ref{tab:main}. }
\label{tab:autoattack}
\end{table*}

We design a group-wise IW approach that groups data into $K$ disjoint subsets according to their TCPR's, and assigns the same weight to a group of data.
For a sample $(x, y)$, the importance weight is $\omega(x, y) = \tilde p_\Beta(\hat r(x, y); a, b)$. The following theorem proves that the expectation of $\ell_{\text{Beta}}$ is an upper bound for the expectation of PSS, which indicates that optimizing $\ell_{\text{Beta}}$ is theoretically beneficial to reducing PSS and CP-efficiency. 
\begin{theorem}
\label{theorem:cost_sensitive_learning_bound_for_CP_PSS_main}
(Learning bound for the expected size of CP prediction sets)
Let
$L_\Beta(f) := \sum_{k=1}^K \sigma_k \cdot \E[ \ell(f(X), Y) | r_f(X,Y) = k ]$, where $\sigma_k \sim p_\Beta(k/(K+1); a, b)$ with $a=1.1, b=5$.
We have the following inequality
\begin{align*}
\E_X [ | \calC_f(X) | ]
\leq L_\Beta(f),
\end{align*}
where $| \calC_f(X) |$ is the cardinality of the prediction set $\calC_f(X)$ for a classifier $f$ with input $X$ and $r_f(X,Y)$ is TCPR of $(X,Y)$ in the classifier $f$.
\end{theorem}
{\bf Remark.} This theorem corroborates our intuition in the previous subsection that optimizing samples with moderate PSS with high importance may lead to the improvement of CP-efficiency. As far as we know, the main theorem is one of the first to build a connection between importance weighting and PSS in conformal prediction. The next section presents our empirical results on various datasets, which further confirm the effectiveness of the proposed AT-UR. See Appendix~\ref{app_sec:proof_thm1} for the full proof. Note that this bound also holds for other $(a,b)$'s with a different constant.

\section{Experiment}
\label{sec:experiment}
We first give the details of our experimental setting and then present the main empirical result. 
\subsection{Experimental Setting}
\label{sec:experiment:setting}
\textbf{Model.} We use the adversarially pre-trained ResNet50 \citep{he2016deep,salman2020adversarially} with $l_{\infty}$ norm and an attack budget $\epsilon_{pt}=4/255$  in all experiment of our paper. The reason is that, besides testing on CIFAR10/100, we also test on more challenging datasets such as Caltech256 and CUB200, on which an adversarially pre-trained model is shown to be much more robust than random initialized weights \citep{liu2023twins}.

\textbf{Dataset.} Four datasets are used to evaluate our method, i.e., CIFAR10, CIFAR100 \citep{krizhevsky2009learning}, Caltech-256 \citep{griffin2007caltech} and Caltech-UCSD Birds-200-2011 (CUB200) \citep{wah2011caltech}. CIFAR10 and CIFAR100 contain low-resolution images of 10 and 100 classes, where the training and validation sets have 50,000 and 10,000 images respectively. Caltech-256 has 30,607 high-resolution images and 257 classes, which are split into a training and a validation set using a 9:1 ratio. CUB200 also contains high-resolution bird images for fine-grained image classification, with 200 classes, 5,994 training images and 5,794 validation images.

\textbf{Training and Adversarial Attack.} In all adversarial training of this paper, we generate adversarial perturbations using PGD attack. The PGD attack has 10 steps, with stepsize $\lambda=2/255$ and attack budget  $\epsilon=8/255$. The batch size is set as 128 and the training epoch is 60. We divide the learning rate by 0.1 at the 30th and 50th epoch. We use the strong AutoAttack \citep{croce2020reliable} with $\epsilon=8/255$ in Tab.~\ref{tab:autoattack}. We use the PGD attack with 100 steps for all other results in this paper. The stepsize and attack budget in PGD100 is the same as in adversarial training, i.e., $\lambda=2/255$ and $\epsilon=8/255$. See more training details in Appendix~\ref{app_sec:exp_details}.

\textbf{Conformal Prediction Setting.} We fix the training set in our experiment and randomly split the original test set into calibration and test set with a ratio of 1:4 for conformal prediction. For each AT method, we repeat the training for three trials with three different seeds and repeat the calibration-test splits five times, which produces 15 trials for our evaluation. The mean and standard deviation of coverage and PSS of 15 trials are reported. If not specified, we use APS \citep{romano2020classification} as the CP method in our paper as the performance of APS is more stable than RAPS, as shown in Fig.~\ref{fig:pitfalls_std_model}. The target coverage is set as 90\% following existing literature in CP \citep{romano2020classification,angelopoulos2020uncertainty,ghosh2023probabilistically}. We use the same adversarial attack setting as in \cite{gendler2021adversarially}, i.e., both calibration and test samples are attacked with the same adversary. We discuss the limitations of this setting in the conclusion section.

\textbf{Baselines.} We use AT \citep{madry2018towards}, Fair-AT (FAT) \citep{xu2021robust} and TRADES \citep{zhang2019theoretically} as the baseline and test the performance of the two proposed uncertainty-reducing methods with the three baselines. AT and TRADES are the most popular adversarial training methods and FAT reduces the robustness variance among classes, which could reduce the PSS, which is validated by our experiment. Note that we only report the performance of CP, i.e., coverage and PSS, in the main paper as the main target of our paper is to improve CP efficiency. 

\begin{figure*}[t]
     \centering
     \includegraphics[width=1.0\textwidth]{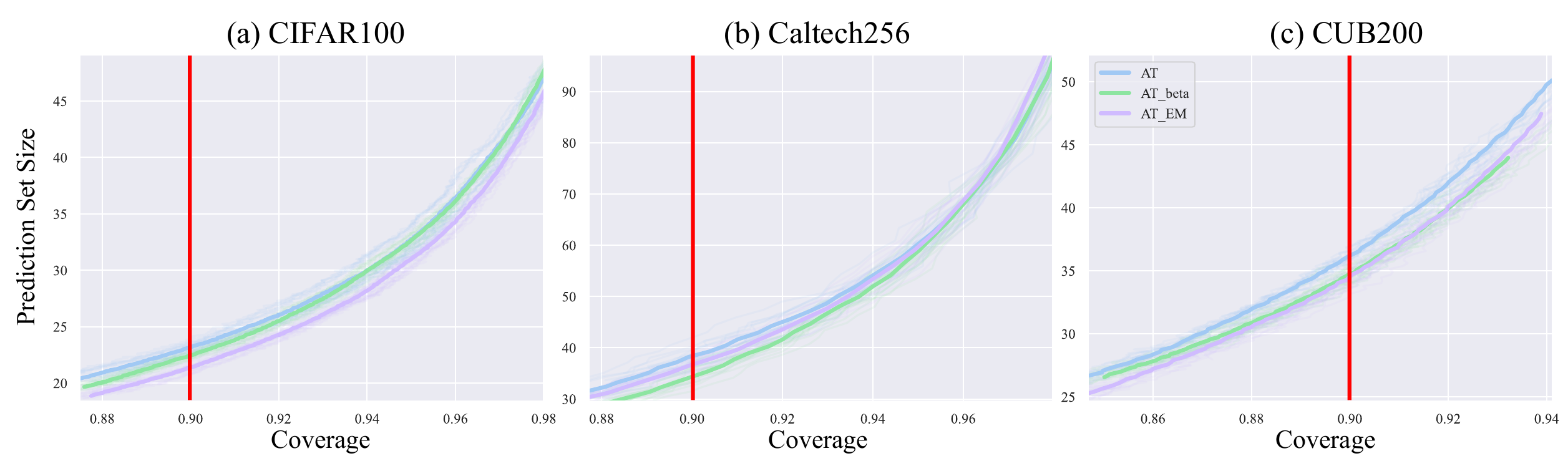}
    \caption{The CP curve of coverage versus PSS. Each point on the curve is obtained by adjusting the threshold $\hat\tau_{\text{cal}}$. We plot 15 CP curves (opaque line) and their average (solid line) for each method. The red vertical line indicates the operating point for 90\% coverage. We visualize the curve with the appropriate y-scale so that the difference is better visualized. }
    \label{fig:cp_curve}
\end{figure*}

\begin{table*}[]
\centering
\centering (1) AT\\
\begin{tabular}{c|c|c|c|c}
\hline
    & AT           & AT-EM        & AT-Beta      & AT-Beta-EM  \\ \hline
Cvg & 90.49(0.65) & 90.19(0.80) & 90.26(0.38) & 90.44(0.61)  \\ \hline
PSS &    23.82(0.94)          &     21.79(1.00)        &       22.85(0.52)      &   22.55(0.88)  \\         \hline           
\end{tabular}
\\
\centering (2) FAT\\
\begin{tabular}{c|c|c|c|c}
\hline
    & FAT          & FAT-EM       & FAT-Beta                                                    & FAT-Beta-EM \\ \hline
Cvg & 90.68(0.56) & 90.13(0.84) & 90.28(0.29) & 90.39(0.72) \\ \hline
PSS &     23.87(0.75)        &       23.72(3.03)   &      23.27(0.40)      &   22.27(0.95)  \\         \hline           
\end{tabular}
\\
\centering (3) TRADES\\
\begin{tabular}{c|c|c|c|c}
\hline
    & TRADES       & TRADES-EM    & TRADES-Beta  & TRADES-Beta-EM \\ \hline
Cvg & 90.42(0.67) & 90.36(0.70) & 90.36(0.68) & 90.13(0.77) \\ \hline
PSS &     27.67(1.08)         &     26.75(1.00)         &     27.56(1.17)      &  26.14(1.16)   \\         \hline           
\end{tabular}
\caption{Coverage and PSS under the uncertainty-aware attack on CIFAR100.}
\label{tab:uncertainty_attack}
\end{table*}

\subsection{Experimental Results}
\textbf{Efficacy of AT-UR in reducing PSS.} The coverage and PSS of all tested methods under the AutoAttack are shown in Tab.~\ref{tab:autoattack}. The proposed AT-UR methods effectively reduce the PSS when combined with the three AT baselines on four datasets, validating our intuition on the connection between the two factors, i.e., predictive entropy and TCPR, and PSS. More importantly, the result is also consistent with our finding in Theorem \ref{theorem:cost_sensitive_learning_bound_for_CP_PSS_main}. There are two phenomena worth noting. First, the Beta weighting generally works better than EM when using AT and FAT, with Beta+EM potentially improving the CP-efficiency in some cases. Second, when using TRADES, EM is more promising than Beta weighting (e.g., EM is better than the other two on three out of four datasets). Thus, we recommend that for AT and FAT, using Beta or Beta-EM is the first choice if one needs to train an adversarially robust and also CP-efficient model, while for TRADES, it is more reasonable to first try EM.
Note that although the Top-1 accuracy of our method (Appendix~\ref{app_sec:exp_results}) is decreased compared to baselines, the main target of our method is to improve CP efficiency as we use the conformal prediction instead of the Top-1 prediction. Tab.~\ref{tab:main_normalized} shows the normalized PSS result to mitigate the influence of different K's on the comparison. The coverage and PSS on clean images are reported in the Appendix~\ref{app_sec:exp_results}. Our AT-UR is effective at improving the CP-efficiency on clean images as well.

\textbf{Coverage-PSS curve visualization.} To visualize the effect of AT-UR more comprehensively, we plot the CP curve by adjusting the threshold $\hat \tau_{\text{cal}}$ to get different points on the curve of coverage versus PSS. Fig.~\ref{fig:cp_curve} shows the CP curve of AT, AT-Beta and AT-EM on three datasets. It demonstrates that AT-UR achieves a reduced PSS compared to the AT baseline, not only at 90\% coverage, but also over a wide range of coverage values.
\subsection{Detailed Empirical Analysis}

\textbf{(a) Sensitivity to hyperparameters and performance under different attack budgets. }We use different Beta-weighting hyperparameters on Caltech256. The performance is stable within a range of 
 b=(3.0, 4.0, 5.0) as shown in Tab.~\ref{tab:hyper_beta}. In addition to the $\epsilon$=8.0, we test different attack budgets $\epsilon$=4.0, 12.0, 16.0 and report the result on Caltech256 in Tab.~\ref{tab:diff_eps}. The result shows that across the attack budgets, our method is consistently better than the AT baseline.

\textbf{(b) Compare with Uncertainty-Aware training \citep{einbinder2022training}.} We train three models with vanilla AT, Conformal AT in \citep{einbinder2022training} and our AT-Beta on CIFAR100 respectively. The experiment follows the setting in the Conformal Training (See more details in Appendix~\ref{app_sec:exp_details}). AT, Conformal AT and AT-Beta have the averaged coverage and PSS of (89.82, 33.43), (90.36, 35.32) and (89.78, \textbf{30.18}) respectively, demonstrating the effectiveness of our Beta-weighting scheme over Conformal AT in the adversarial environment.

\textbf{(c) Does Focal loss improve CP-efficiency?} We consider using a power function $\hat r_i^{\eta}$ as in focal loss \citep{lin2017focal} to generate loss weights and test the CP performance of AT-Focal. We set $\eta=0.5$ based on a hyperparameter search from $\{0.1,0.5,1.0,2.0\}$. AT-Focal forces the model to focus on hard samples, contrary to our AT-Beta which focuses on promising samples. The averaged coverage and PSS of AT-Focal on CIFAR100 and Caltech256 are (90.50, 27.24) and (91.38, 48.35) respectively, which is far worse than the AT baseline of (90.45, 23.79) and (91.35, 43.20). This result corroborates that promising samples are crucial for improving CP-efficiency instead of hard samples. 

\textbf{(d) What is the difference between label smoothing and AT-EM?} The formulation of AT-EM is similar to the formulation of label smoothing \citep{muller2019does}, if we combine the log term in (\ref{eqn:at_em}). However, label smoothing and AT-EM train the model towards two different directions: the former increases the prediction entropy (by smoothing the label probabilities to be more uniform), while the latter decreases the prediction entropy. We validate this argument on Caltech256 and find that label smoothing makes the CP-efficiency much worse than the AT baseline, with an averaged coverage and PSS of (90.22, 46.39), compared to (91.35, 43.20) of AT.

\textbf{(e) Is AT-UR robust to uncertainty-aware adversarial attacks?} PGD attack and AutoAttack are both designed for reducing Top-1 accuracy instead of CP-efficiency. We design an uncertainty adversarial attack that maximizes the entropy of the predictive distribution to test the performance of our AT-UR under the uncertainty-aware adversarial attack. Tab.~\ref{tab:uncertainty_attack} shows the result of our method when combined with three AT methods on CIFAR100 using the uncertainty-aware attack in the inference stage (using the same well-trained models with the Top-1 attack). The attacker is PGD100 and all settings here are the same as in Sec.~\ref{sec:experiment:setting}. Our method is still competitive under this uncertainty-aware adversarial attack.

\section{Conclusion}
This paper first studies the pitfalls of CP under adversarial attacks and thus underscores the importance of AT when using CP in an adversarial environment. Then we unveil the compromised CP-efficiency of popular AT methods and propose to design uncertainty-reducing AT for CP-efficiency based on our empirical observation on two factors affecting the PSS. Our theoretical results establish the connection between PSS and Beta weighting. Our experiment validates the effectiveness of the proposed AT-UR on four datasets when combined with three AT baselines. A common limitation shared by this study and \cite{gendler2021adversarially} is the assumption that the adversarial attack is known, enabling the calibration set to be targeted by the same adversary as the test set. In future research, we will alleviate this constraint by exploring CP within an adversary-agnostic context. We will also explore the robustness of conformal prediction in large language models.

\section*{Impact Statement}
This paper investigating and improving CP-efficiency for deep learning models under adversarial attacks makes an important contribution to the reliability and safety of artificial intelligence (AI) systems. The theoretical and empirical results in this paper hold immense societal implications, particularly in high-stakes applications such as self-driving cars and medical diagnosis, advancing the positive impact of AI on society by promoting secure and reliable AI-driven advancements.

\section*{Acknowledgement}
This work was supported by a grant from the Research Grants Council of the Hong Kong Special Administrative Region, China (Project No. CityU 11215820).

\bibliography{ref}
\bibliographystyle{icml2024}

\newpage
\appendix
\onecolumn
\clearpage
\appendix

\thispagestyle{empty}

%
\onecolumn 
\section{Adaptive Prediction Sets \citep{romano2020classification}}
\label{app_sec:aps}
We introduce one example of prediction set function, i.e., APS conformal prediction used in our experiment. Assume we have the prediction distribution $\pi(x)=f_{\theta}(x)$ and order this probability vector with the descending order $\pi_{(1)}(x) \ge \pi_{(2)}(x) \geq \ldots \geq \pi_{(K)}(x)$. We first define the following generalized conditional quantile function,
\begin{align} \label{eqn:gcq}
  Q(x; \pi, \tau) & = \min \{ k \in \{1,\ldots,K\} \ : \ \pi_{(1)}(x) + \pi_{(2)}(x) + \ldots + \pi_{(k)}(x) \geq \tau \},
\end{align}
which returns the class index with the generalized quantile $\tau\in[0,1]$. The function $\gS$ can be defined as
\begin{align} \label{eqn:S}
    \mathcal{S}(x, u ; \pi, \tau) & = 
    \begin{cases}
    \text{ `$y$' indices of the $Q(x ; \pi,\tau)-1$ largest $\pi_{y}(x)$},
    & \text{ if } u \leq U(x ; \pi,\tau) , \\
    \text{ `$y$' indices of the $Q(x ; \pi,\tau)$ largest $\pi_{y}(x)$},
    & \text{ otherwise},
    \end{cases}
\end{align}
where
\begin{align*}
    U(x; \pi, \tau) & =  \frac{1}{\pi_{(Q(x ; \pi, \tau))}(x)} \left[\sum_{k=1}^{Q(x ; \pi, \tau)} \pi_{(k)}(x) - \tau \right].
\end{align*}
It has input $x$, $u \in [0,1]$, $\pi$, and $\tau$ and can be seen as a generalized inverse of Equation~\ref{eqn:gcq}. 

On the calibration set, we compute a generalized inverse quantile conformity score with the following function,
\begin{align}
    E(x,y,u;\pi) & = \min \left\{ \tau \in [0,1] : y \in \gS(x, u ; \pi, \tau) \right\},
\end{align}
which is the smallest quantile to ensure that the ground-truth class is contained in the prediction set $\gS(x, u ; \pi, \tau)$. With the conformity scores on calibration set $\{E_i\}_{i=1}^{n_c}$, we compute the $\ceil{(1-\alpha)(1+n_c)}$th largest value in the score set as $\hat\tau_{\text{cal}}$. During inference, the prediction set is generated with $\gS(x^*, u ; \pi^*, \hat\tau_{\text{cal}})$ for a novel test sample $x^*$.

\section{More Experimental Details}
\label{app_sec:exp_details}
\paragraph{APS Setting.} We use the default setting of APS specified in the official code of \cite{angelopoulos2020uncertainty}, i.e., first use temperature scaling \citep{platt1999probabilistic,guo2017calibration} to calibrate the prediction distribution then compute the generalized inverse quantile conformity score to perform the calibration and conformal prediction. 
\paragraph{Hyperparameter and Baseline Setting.} As mentioned in the main paper, we use $a=1.1$ and search $b$ from the discrete set $\{2.0,3.0,4.0,5.0\}$ in Beta distribution since the parameter combinations perform well in our pilot study and satisfy the goal of focusing on promising samples. The learning rate and weight decay of AT, FAT and TRADES are determined by grid search from \{1e-4,3e-4,1e-3,3e-3,1e-2\} and \{1e-3,1e-4,1e-5\} respectively. We compute the class weight for FAT using the output of a softmax function with error rate of each class as input. The temperature in the softmax function is set as 1.0. For TRADES, we follow the default setting $\beta=6.0$ for the KL divergence term \citep{zhang2019theoretically}. Our AT-UR method also determines the learning rate and weight decay using the grid search with the same mentioned grid. For TRADES, we weight both the cross-entropy loss and KL divergence loss with the Beta density function based on TCPR. 

\paragraph{CP Curve.} The CP curve in Fig.~\ref{fig:cp_curve} is obtained by using different threshold values, for instance, using the linspace function in numpy \citep{harris2020array} with \code{np.linspace(0.9,1.1,200)}$\times\hat\tau_{\text{cal}}$ generates 200 different (coverage, PSS) points.

\paragraph{Compare with Conformal AT \citep{einbinder2022training}.} We use the experimental setting in the original paper, where they train a randomly initialized ResNet50 using SGDM with batch size=128, learning rate=0.1, weight decay=0.0005, for 120 epochs, where the learning rate is divided by 10 at 100th epoch. 45000 original training samples are used for training and the remaining 5000 samples are used as a held-out set for computing the conformal loss. We use the same attack parameters in this setting as in other experiments during both training and inference. In the conformal inference based on APS scores, we split the original test set with a ratio of 1:1 into a calibration and a test set and only test the final-epoch model. We run three trials for each approach and report the average coverage and PSS in the main paper. The experiment result that the effectiveness of our AT-Beta is generalizable to the randomly intialized model training.

\begin{table*}[t]
\centering
\resizebox{\linewidth}{!}{
\begin{tabular}{|c|cc|cc|cc|cc|}
\hline
     Dataset   & \multicolumn{2}{c|}{CIFAR10}     & \multicolumn{2}{c|}{CIFAR100}    & \multicolumn{2}{c|}{Caltech256}  & \multicolumn{2}{c|}{CUB200}      \\ \hline 
     Metric   & \multicolumn{1}{c|}{Cvg} & PSS & \multicolumn{1}{c|}{Cvg} & PSS & \multicolumn{1}{c|}{Cvg} & PSS & \multicolumn{1}{c|}{Cvg} & PSS \\ \hline
AT      & \multicolumn{1}{c|}{90.55(0.51)}     &   3.10(0.07)   & \multicolumn{1}{c|}{90.45(0.59) }     &  23.79(0.80)    & \multicolumn{1}{c|}{91.35(0.85)}     &   43.20(2.11)   & \multicolumn{1}{c|}{90.33(0.89) }     &   37.37(2.11)   \\ 
AT-EM$^*$   & \multicolumn{1}{c|}{90.39(0.48) }     &  \textbf{3.05(0.05)}    & \multicolumn{1}{c|}{90.35(0.82) }     & \textbf{22.05(1.02)}     & \multicolumn{1}{c|}{91.09(0.79) }     &  41.42(2.52)    & \multicolumn{1}{c|}{90.08(1.10) }     &   34.77(2.60)   \\ 
AT-Beta$^*$ & \multicolumn{1}{c|}{90.46(0.51) }     &   3.11(0.07)   & \multicolumn{1}{c|}{90.10(0.51)  }     &  22.64(0.65)    & \multicolumn{1}{c|}{90.20(0.84) }     &   \textbf{35.39(2.66)}   & \multicolumn{1}{c|}{90.17(1.06)}     &   35.25(2.15)   \\ 
AT-Beta-EM$^*$   & \multicolumn{1}{c|}{90.65(0.62) }     &  3.10(0.08)    & \multicolumn{1}{c|}{90.40(0.60)}     & 22.53(0.91)     & \multicolumn{1}{c|}{90.81(1.00)}     &   36.17(3.73)   & \multicolumn{1}{c|}{90.31(0.84) }     &  \textbf{33.10(1.74)}    \\ \hline
FAT   & \multicolumn{1}{c|}{90.69(0.61)}     &   3.16(0.07)  & \multicolumn{1}{c|}{90.41(0.67) }     & 23.54(0.81)     & \multicolumn{1}{c|}{90.70(0.77)}     & 41.52(2.43)      & \multicolumn{1}{c|}{ 90.50(1.17)}     &  39.43(2.88)    \\
FAT-EM$^*$  & \multicolumn{1}{c|}{ 90.54(0.68) }     &   3.06(0.06)  & \multicolumn{1}{c|}{90.00(0.82)}     & 23.47(2.71)     & \multicolumn{1}{c|}{ 90.55(0.79) }     & 39.72(2.49)     & \multicolumn{1}{c|}{89.89(0.92) }     &  35.51(2.06)     \\
FAT-Beta$^*$   & \multicolumn{1}{c|}{90.47(0.51)}     &   3.16(0.08)   & \multicolumn{1}{c|}{90.22(0.47)}     &  23.15(0.71)     & \multicolumn{1}{c|}{89.90(0.70)}     &  34.72(2.25)    & \multicolumn{1}{c|}{ 89.92(0.84) }     &  35.46(1.71)     \\
FAT-Beta-EM$^*$   & \multicolumn{1}{c|}{90.71(0.61)}     &  \textbf{3.04(0.07)}   & \multicolumn{1}{c|}{90.36(0.50) }     & \textbf{22.28(0.63)}     & \multicolumn{1}{c|}{90.41(0.61) }     &   \textbf{33.59(2.75)}    & \multicolumn{1}{c|}{ 89.88(0.91) }     &  \textbf{34.35(1.68)}    \\
\hline
TRADES   & \multicolumn{1}{c|}{90.72(0.62) }     & 3.31(0.09)     & \multicolumn{1}{c|}{90.35(0.57) }     &  27.60(0.97)    & \multicolumn{1}{c|}{90.82(0.81) }     &  44.80(3.42)    & \multicolumn{1}{c|}{ 90.38(0.76) }     &  52.18(2.60)     \\ 
TRADES-EM$^*$   & \multicolumn{1}{c|}{90.54(0.40) }     &  \textbf{3.16(0.05)}   & \multicolumn{1}{c|}{90.36(0.71) }     & \textbf{26.76(1.00)}     & \multicolumn{1}{c|}{ 90.68(0.87) }     & \textbf{38.83(3.78)}     & \multicolumn{1}{c|}{ 90.05(0.76)}     & \textbf{44.96(2.75)}      \\
TRADES-Beta$^*$  & \multicolumn{1}{c|}{90.41(0.56) }     &   3.30(0.09)  & \multicolumn{1}{c|}{90.14(0.85)}     &  27.22(1.42)    & \multicolumn{1}{c|}{90.48(0.70) }     &  38.94(2.74)    & \multicolumn{1}{c|}{89.83(0.84) }     &  49.63(2.59)    \\
TRADES-Beta-EM$^*$   & \multicolumn{1}{c|}{90.01(0.40)}     &  3.21(0.06)   & \multicolumn{1}{c|}{90.54(0.24)}     &   26.55(0.35)   & \multicolumn{1}{c|}{90.52(0.74) }     &  39.83(3.02)    & \multicolumn{1}{c|}{90.16(1.12)}     & 48.54(2.58) \\
\hline
\end{tabular}
}
\vspace{-0.3cm}
\caption{Comparison of AT baselines and the proposed AT-UR variants denoted with $^*$, under the PGD100 attack. }
\vspace{-0.5cm}
\label{tab:main}
\end{table*}

\section{More Experimental Results}
\label{app_sec:exp_results}
Note that this paper uses CP as the inference method to achieve a coverage guarantee, which is orthogonal to the Top-1 inference method. Thus, Top-1 accuracy is not a relevant metric in the context of CP inference. Nevertheless, we show the Top-1 accuracy of tested methods in Tab.~\ref{tab:top_1_acc}. Using AT-UR generally worsens the Top-1 accuracy, especially for TRADES. However, note that using TRADES-Beta-EM can improve the Top-1 robust accuracy of TRADES-Beta on CIFAR10 and TRADES-EM on Caltech256. This result again confirms the observation that Top-1 accuracy is not necessarily correlated with CP-efficiency. \revise{When we compare the result of using PGD100 and AA, the robust accuracy under AA drops while the prediction set size reduced (CP efficiency is improved), indicating that a stronger attack can lead to reduced PSS. }To reduce the effect of number of classes (K) on the PSS, Tab.~\ref{tab:main_normalized} shows the normalized PSS using K when using the PGD100 attack.

\begin{figure}[t]
     \centering
     \includegraphics[width=\textwidth]{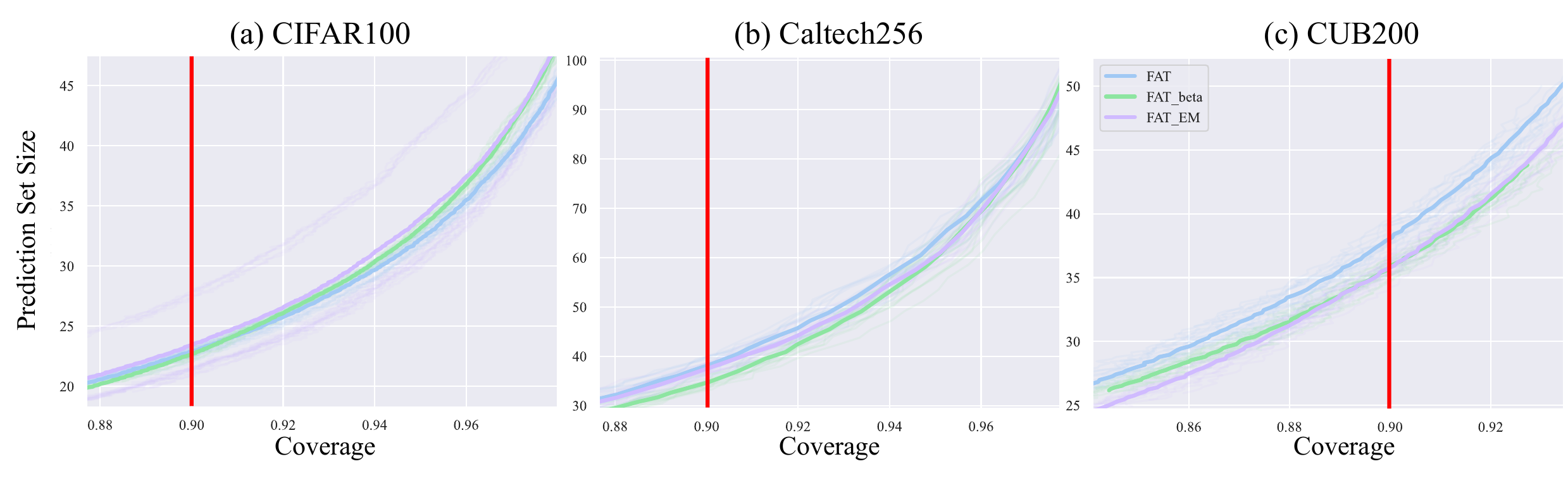}
    \caption{The CP curve of coverage versus prediction set size using FAT and PGD100 attack.}
    \label{fig:cp_curve_fat}
\end{figure}

\begin{figure}[h]
     \centering
     \includegraphics[width=\textwidth]{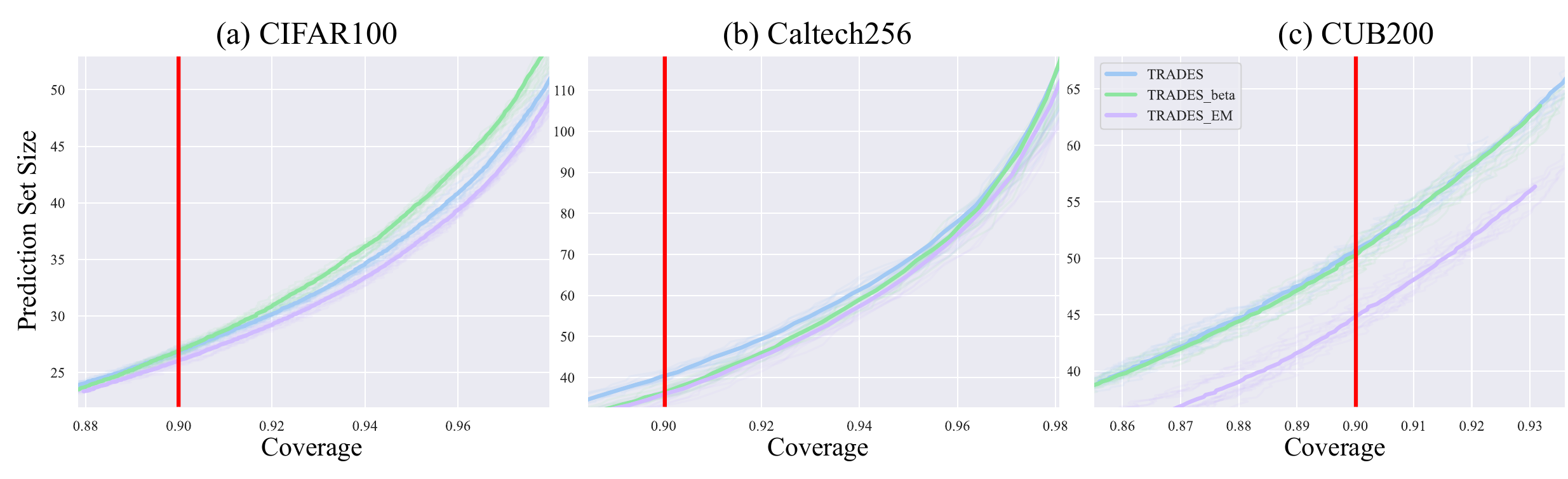}
    \caption{The CP curve of coverage versus prediction set size using TRADES and PGD100 attack.}
    \label{fig:cp_curve_trades}
\end{figure}

Fig.~\ref{fig:cp_curve_fat} and Fig.~\ref{fig:cp_curve_trades} shows the CP curve of FAT and TRADES when they are combined with EM and Beta on three datasets. It demonstrates that the CP-efficiency is also improved when using FAT and TRADES as in the experiment using AT. In most cases (5 out of 6), AT-UR (either EM or Beta) has a lower PSS than the corresponding baseline within a large range of coverage. 

Fig.~\ref{fig:tcpr_percentage} shows the percentage of samples TCPR$=$1, 1$<$TCPR$<$20 and TCPR$\geq$20 during AT on CIFAR100, demonstrating that the promising samples are the majority in most time training, especially for the early 30 epochs.

We include the Coverage and PSS on clean images in Tab.~\ref{tab:main_clean_image} as a reference. It shows that our AT-UR improves the CP-efficiency even on clean images across various datasets and adversarial training methods.

\begin{table*}[t]
\centering
\resizebox{\linewidth}{!}{
\begin{tabular}{|c|cc|cc|cc|cc|}
\hline
     Dataset   & \multicolumn{2}{c|}{CIFAR10}     & \multicolumn{2}{c|}{CIFAR100}    & \multicolumn{2}{c|}{Caltech256}  & \multicolumn{2}{c|}{CUB200}      \\ \hline 
     Metric   & \multicolumn{1}{c|}{Cvg} & PSS & \multicolumn{1}{c|}{Cvg} & PSS & \multicolumn{1}{c|}{Cvg} & PSS & \multicolumn{1}{c|}{Cvg} & PSS \\ \hline
AT      & \multicolumn{1}{c|}{97.42(0.16) }     &   1.62(0.03)   & \multicolumn{1}{c|}{93.32(0.74) }     &  6.19(0.42)  & \multicolumn{1}{c|}{95.64(0.68)}     &   13.89(1.75)  & \multicolumn{1}{c|}{92.87(0.53) }     &   8.45(0.49) \\ 
AT-EM$^*$   & \multicolumn{1}{c|}{97.10(0.13)} & 1.52(0.02)& \multicolumn{1}{c|}{93.23(0.71)} & 6.21(0.43)& \multicolumn{1}{c|}{95.19(0.55)} & 11.76(1.22)& \multicolumn{1}{c|}{92.72(0.58)} & 8.44(0.43)  \\ 
AT-Beta$^*$ & \multicolumn{1}{c|}{96.83(0.13)} & 1.55(0.01)&\multicolumn{1}{c|}{92.71(0.74)} & 5.21(0.31)&\multicolumn{1}{c|}{95.01(0.57)} & 9.84(0.88)	& \multicolumn{1}{c|}{92.34(0.64)} & 7.40(0.41)  \\ 
AT-Beta-EM$^*$   & \multicolumn{1}{c|}{96.80(0.17)} & 1.53(0.02)&\multicolumn{1}{c|}{92.87(0.72)} & 5.34(0.35)&\multicolumn{1}{c|}{94.87(0.42)} & 10.21(0.80)&\multicolumn{1}{c|}{92.59(0.62)} & 7.83(0.42)
   \\ \hline
FAT   & \multicolumn{1}{c|}{97.38(0.15)} & 1.60(0.03)&	\multicolumn{1}{c|}{93.10(0.73)} & 5.98(0.41)&	\multicolumn{1}{c|}{95.41(0.62)} & 12.80(1.13)&	\multicolumn{1}{c|}{92.70(0.31)} & 8.38(0.26)   \\
FAT-EM$^*$  & \multicolumn{1}{c|}{97.09(0.16)} & 1.51(0.02)&\multicolumn{1}{c|}{93.19(0.72)} & 5.94(0.44)&\multicolumn{1}{c|}{95.03(0.62)} & 11.75(1.16)&	\multicolumn{1}{c|}{92.62(0.48)} & 8.16(0.38)    \\
FAT-Beta$^*$   &\multicolumn{1}{c|}{97.00(0.10)} & 1.55(0.01)&	\multicolumn{1}{c|}{92.50(0.70)} & 5.02(0.30)&\multicolumn{1}{c|}{94.89(0.48)} & 9.79(0.88)&\multicolumn{1}{c|}{92.14(0.49)} & 7.11(0.32)  \\
FAT-Beta-EM$^*$   & \multicolumn{1}{c|}{96.68(0.13)} & 1.52(0.01)&\multicolumn{1}{c|}{92.72(0.70)} & 5.27(0.30)&\multicolumn{1}{c|}{94.87(0.44)} & 9.66(0.72)&\multicolumn{1}{c|}{92.50(0.44)} & 7.65(0.29)
  \\
\hline
TRADES   & \multicolumn{1}{c|}{96.43(0.16)} & 1.72(0.02)&\multicolumn{1}{c|}{91.84(0.81)} & 8.17(0.58)&\multicolumn{1}{c|}{94.18(0.47)} & 17.39(1.28)&\multicolumn{1}{c|}{91.56(0.66)} & 14.22(0.79)    \\ 
TRADES-EM$^*$   & \multicolumn{1}{c|}{96.04(0.34)} & 1.62(0.03)&\multicolumn{1}{c|}{91.25(0.71)} & 10.04(0.55)&\multicolumn{1}{c|}{93.84(0.69)} & 13.39(1.49)&\multicolumn{1}{c|}{91.63(0.94)} & 14.43(1.99)   \\
TRADES-Beta$^*$  & \multicolumn{1}{c|}{96.05(0.24)} & 1.67(0.03)&\multicolumn{1}{c|}{91.41(0.83)} & 7.37(0.48)&\multicolumn{1}{c|}{93.35(0.60)} & 11.68(0.92)&\multicolumn{1}{c|}{90.83(0.74)} & 10.72(0.74)  \\
TRADES-Beta-EM$^*$   & \multicolumn{1}{c|}{95.66(0.10)} & 1.66(0.01)&\multicolumn{1}{c|}{91.30(0.81)} & 7.89(0.53)&\multicolumn{1}{c|}{93.29(0.55)} & 13.13(0.91)&\multicolumn{1}{c|}{90.79(1.03)} & 11.40(0.96) \\
\hline
\end{tabular}
}
\vspace{-0.3cm}
\caption{Comparison of AT baselines and the proposed AT-UR variants denoted with $^*$ on clean images. }
\vspace{-0.5cm}
\label{tab:main_clean_image}
\end{table*}

\begin{figure}[h]
     \centering
     \includegraphics[width=0.6\textwidth]{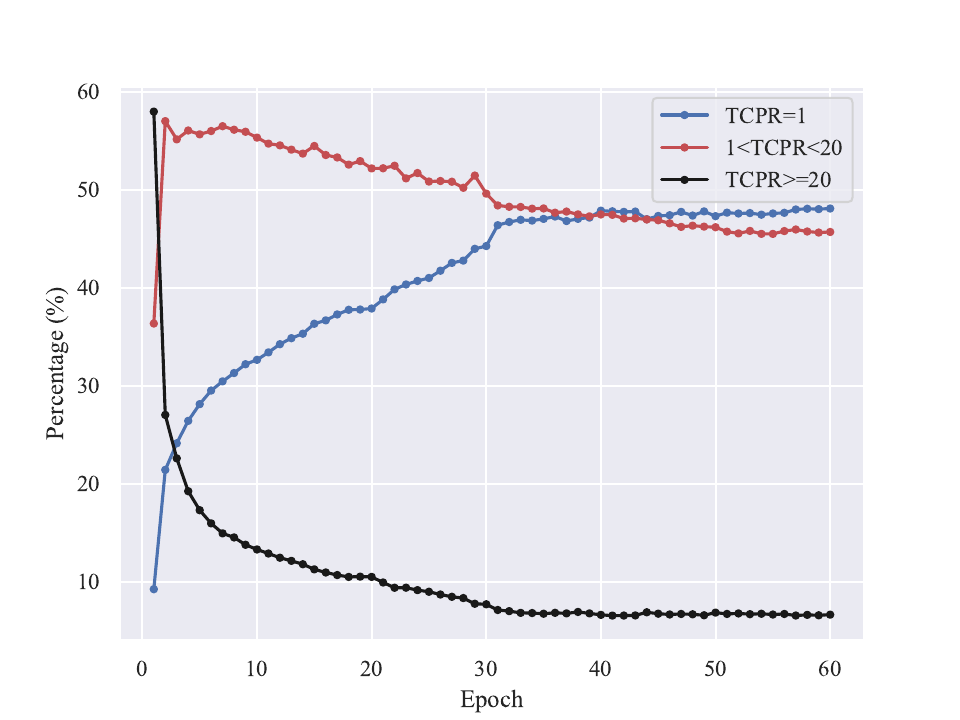}
    \caption{The percentage of samples with different TCPR's in CIFAR100 training.}
    \label{fig:tcpr_percentage}
\end{figure}

\begin{table}[t]
\centering
\resizebox{\linewidth}{!}{
\begin{tabular}{|c|cc|cc|cc|cc|}
\hline
     Dataset   & \multicolumn{2}{c|}{CIFAR10}     & \multicolumn{2}{c|}{CIFAR100}    & \multicolumn{2}{c|}{Caltech256}  & \multicolumn{2}{c|}{CUB200}      \\ \hline 
     Metric   & \multicolumn{1}{c|}{Cvg} & NPSS & \multicolumn{1}{c|}{Cvg} & NPSS & \multicolumn{1}{c|}{Cvg} & NPSS & \multicolumn{1}{c|}{Cvg} & NPSS \\ \hline
AT      & \multicolumn{1}{c|}{90.55(0.51)}     &   31.0(0.7)   & \multicolumn{1}{c|}{90.45(0.59) }     &  23.79(0.80)    & \multicolumn{1}{c|}{91.35(0.85)}     &   16.8(0.8)   & \multicolumn{1}{c|}{90.33(0.89) }     &   18.7(1.1)   \\ 
AT-EM$^*$   & \multicolumn{1}{c|}{90.39(0.48) }     &  \textbf{30.5(0.5)}    & \multicolumn{1}{c|}{90.35(0.82) }     & \textbf{22.05(1.02)}     & \multicolumn{1}{c|}{91.09(0.79) }     &  16.1(1.0)    & \multicolumn{1}{c|}{90.08(1.10) }     &   17.4(1.3)   \\ 
AT-Beta$^*$ & \multicolumn{1}{c|}{90.46(0.51) }     &   31.1(0.7)   & \multicolumn{1}{c|}{90.10(0.51)  }     &  22.64(0.65)    & \multicolumn{1}{c|}{90.20(0.84) }     &   \textbf{13.8(1.0)}   & \multicolumn{1}{c|}{90.17 (1.06)}     &   17.6(1.1)   \\ 
AT-Beta-EM$^*$   & \multicolumn{1}{c|}{90.65(0.62) }     &  31.0(0.8)    & \multicolumn{1}{c|}{90.40(0.60)}     & 22.53(0.91)     & \multicolumn{1}{c|}{90.81(1.00)}     &   14.1(1.4)   & \multicolumn{1}{c|}{90.31(0.84) }     &  \textbf{16.5(0.9)}    \\ \hline
FAT   & \multicolumn{1}{c|}{90.69(0.61)}     &   31.6(0.7)  & \multicolumn{1}{c|}{90.41(0.67) }     & 23.54(0.81)     & \multicolumn{1}{c|}{90.70(0.77)}     & 16.2(0.9)      & \multicolumn{1}{c|}{ 90.50(1.17)}     &  19.7(1.4)    \\
FAT-EM$^*$  & \multicolumn{1}{c|}{ 90.54(0.68) }     &   30.6(0.6)  & \multicolumn{1}{c|}{90.00(0.82)}     & 23.47(2.71)     & \multicolumn{1}{c|}{ 90.55(0.79) }     & 15.5(1.0)     & \multicolumn{1}{c|}{89.89(0.919) }     &  17.8(1.0)     \\
FAT-Beta$^*$   & \multicolumn{1}{c|}{90.47(0.51)}     &   31.6(0.8)   & \multicolumn{1}{c|}{90.22(0.47)}     &  23.15(0.71)     & \multicolumn{1}{c|}{89.90(0.70)}     &  13.5(0.9)    & \multicolumn{1}{c|}{ 89.92(0.84) }     &  17.7(0.9)     \\
FAT-Beta-EM$^*$   & \multicolumn{1}{c|}{90.71(0.61)}     &  \textbf{30.4(0.7)}   & \multicolumn{1}{c|}{90.36(0.50) }     & \textbf{22.28(0.63)}     & \multicolumn{1}{c|}{90.41(0.61) }     &   \textbf{13.1(1.1)}    & \multicolumn{1}{c|}{ 89.88(0.91) }     &  \textbf{17.2(0.8)}    \\
\hline
TRADES   & \multicolumn{1}{c|}{90.72(0.62) }     & 33.1(0.9)     & \multicolumn{1}{c|}{90.35(0.57) }     &  27.60(0.97)    & \multicolumn{1}{c|}{90.82(0.81) }     &  17.4(1.3)    & \multicolumn{1}{c|}{ 90.38(0.76) }     &  26.1(1.3)     \\ 
TRADES-EM$^*$   & \multicolumn{1}{c|}{90.54(0.40) }     &  \textbf{31.6(0.5)}   & \multicolumn{1}{c|}{90.36(0.71) }     & \textbf{26.76(1.00)}     & \multicolumn{1}{c|}{ 90.68(0.87) }     & \textbf{15.1(1.5)}     & \multicolumn{1}{c|}{ 90.05(0.76)}     & \textbf{22.5(1.4)}      \\
TRADES-Beta$^*$  & \multicolumn{1}{c|}{90.41(0.56) }     &   33.0(0.9)  & \multicolumn{1}{c|}{90.14(0.85)}     &  27.22(1.42)    & \multicolumn{1}{c|}{90.48(0.70) }     &  15.1(1.1)    & \multicolumn{1}{c|}{89.83(0.84) }     &  24.8(1.3)    \\
TRADES-Beta-EM$^*$   & \multicolumn{1}{c|}{90.01(0.40)}     &  32.1(0.6)   & \multicolumn{1}{c|}{90.54(0.24)}     &   26.55(0.35)   & \multicolumn{1}{c|}{90.52(0.74) }     &  15.5(1.2)    & \multicolumn{1}{c|}{90.16(1.12)}     & 24.3(1.3) \\
\hline
\end{tabular}
}
\vspace{-0.3cm}
\caption{\revise{Comparison of AT baselines and the proposed AT-UR variants denoted with $^*$, under the PGD100 attack. The average coverage (Cvg) and prediction set size normalized by the class number (NPSS, \%) are presented. } }
\label{tab:main_normalized}
\end{table}

\begin{table}[t]
\centering
\vspace{-0.3cm}
\resizebox{\linewidth}{!}{
\begin{tabular}{|c|cc|cc|cc|cc|}
\hline
     Dataset   & \multicolumn{2}{c|}{CIFAR10}     & \multicolumn{2}{c|}{CIFAR100}    & \multicolumn{2}{c|}{Caltech256}  & \multicolumn{2}{c|}{CUB200}      \\ \hline 
     Metric   & \multicolumn{1}{c|}{Std. Acc.} & Rob. Acc. & \multicolumn{1}{c|}{Std. Acc.} & Rob. Acc. & \multicolumn{1}{c|}{Std. Acc.} & Rob. Acc. & \multicolumn{1}{c|}{Std. Acc.} & Rob. Acc. \\ \hline
AT      & \multicolumn{1}{c|}{ 89.76(0.15)  }     &  50.17(0.91)    & \multicolumn{1}{c|}{68.92(0.38) }     &  28.49(1.14)    & \multicolumn{1}{c|}{ 75.28(0.51) }     &  47.53(0.67)  & \multicolumn{1}{c|}{65.36(0.27) }     &  26.29(0.45)   \\ 
AT-EM$^*$   & \multicolumn{1}{c|}{ 90.02(0.10) }     &  48.92(0.39)   & \multicolumn{1}{c|}{68.39(0.51) }     &  28.33(0.73)     & \multicolumn{1}{c|}{74.62(0.22)   }     &  46.23(0.44)   & \multicolumn{1}{c|}{ 64.75(0.33) }     &    25.60(0.25)  \\ 
AT-Beta$^*$ & \multicolumn{1}{c|}{ 89.81(0.22)}     &   47.50(0.78)  & \multicolumn{1}{c|}{68.50(0.28) }     &  28.04(0.57)   & \multicolumn{1}{c|}{ 74.66(0.54)  }     &   45.40(0.59)  & \multicolumn{1}{c|}{ 64.62(0.17)  }     &  25.57(0.41)   \\ 
AT-Beta-EM$^*$   & \multicolumn{1}{c|}{ 90.00(0.06)}     &  46.69(0.71)   & \multicolumn{1}{c|}{ 68.45(0.35)}     &    27.20(1.08)   & \multicolumn{1}{c|}{74.71(0.36)}     &  44.88(0.60)     & \multicolumn{1}{c|}{ 64.44(0.22) }     &  25.32(0.38) \\ \hline
FAT      & \multicolumn{1}{c|}{ 89.96(0.25)}     &   49.12(0.70)   & \multicolumn{1}{c|}{68.80(0.38)}     &    28.97(0.53)   & \multicolumn{1}{c|}{75.20(0.33) }     &  47.09(0.70)  & \multicolumn{1}{c|}{ 65.01(0.19) }     &  25.21(0.57)   \\ 
FAT-EM$^*$   & \multicolumn{1}{c|}{ 90.19(0.07)}     &  48.31(0.80)    & \multicolumn{1}{c|}{68.76(0.39)  }     & 25.84(3.54)   & \multicolumn{1}{c|}{74.59(0.20) }     &  45.53(0.70)  & \multicolumn{1}{c|}{ 65.13(0.25)}     & 24.92(0.27)    \\ 
FAT-Beta$^*$ & \multicolumn{1}{c|}{90.07(0.16) }     &   47.09(0.50)  & \multicolumn{1}{c|}{ 68.58(0.33) }     &  27.90(0.56)   & \multicolumn{1}{c|}{74.00(0.39)  }     &  45.12(0.90) & \multicolumn{1}{c|}{ 64.38(0.25)}     &  24.76(0.26)    \\ 
FAT-Beta-EM$^*$   & \multicolumn{1}{c|}{89.86(0.06) }     &  48.61(0.13)   & \multicolumn{1}{c|}{67.95(0.24)}     &   28.06(0.24)    & \multicolumn{1}{c|}{ 74.78(0.10) }     &  45.75(0.48)     & \multicolumn{1}{c|}{ 64.32(0.21)}     &  23.57(0.14)   \\ \hline
TRADES      & \multicolumn{1}{c|}{ 87.31(0.27)}     &   53.07(0.23)   & \multicolumn{1}{c|}{62.83(0.33)  }     &  32.07(0.20)     & \multicolumn{1}{c|}{69.57(0.25)  }     &  47.07(0.37)  & \multicolumn{1}{c|}{ 58.16(0.38)}     &  27.82(0.23)   \\ 
TRADES-EM$^*$   & \multicolumn{1}{c|}{  86.68(0.06)}     &  52.71(0.26)   & \multicolumn{1}{c|}{ 57.03(0.31) }     & 30.29(0.25)     & \multicolumn{1}{c|}{ 57.17(0.39)}     &  39.56(0.59)    & \multicolumn{1}{c|}{45.50(5.81)}     &   22.26(2.26)  \\ 
TRADES-Beta$^*$ & \multicolumn{1}{c|}{ 89.81(0.22)}     &  47.50(0.78)   & \multicolumn{1}{c|}{62.61(0.36)  }     &  30.20(0.31)   & \multicolumn{1}{c|}{ 70.96(0.25)  }     & 46.74(0.23)   & \multicolumn{1}{c|}{ 57.72(0.24)}     &  23.49(0.21)   \\ 
TRADES-Beta-EM$^*$   & \multicolumn{1}{c|}{86.99(0.10) }     &   51.85(0.23)   & \multicolumn{1}{c|}{62.13(0.34)}     &   30.52(0.20)  & \multicolumn{1}{c|}{69.44(0.25)}     &   46.24(0.39)    & \multicolumn{1}{c|}{ 56.03(0.15) }     &  22.90(0.32)  \\ \hline
\end{tabular}
}
\caption{Top-1 clean and robust accuracy comparison of AT baselines and the proposed AT-UR variants under the PGD100 attack.}
\vspace{-0.5cm}
\label{tab:top_1_acc}
\end{table}

\begin{table}[t]
\resizebox{\linewidth}{!}{
\begin{tabular}{|c|c|c|c|c|c|c|}
\hline
    Method & (a,b)=(1.1, 2.0) & (a,b)=(1.1, 3.0) & (a,b)=(1.1, 4.0) & (a,b)=(1.1, 5.0) & (a,b)=(1.1, 6.0) & AT, (a,b)=(1.0, 1.0) \\ \hline
Cvg & 90.59(0.56)   & 90.18(0.85)   & 90.25(0.76)   & 90.20(0.84)   & 90.97(1.08)   & 91.35(0.85)          \\ \hline
PSS & 37.09(2.24)   & 35.22(2.76)   & 35.55(2.28)   & 35.39(2.66)   & 38.54(3.47)   & 43.20(2.11)          \\ \hline
\end{tabular}
}
\caption{Comparison of using different (a,b) in AT-Beta on Caltech256.}
\label{tab:hyper_beta}
\end{table}

\CUT{
\begin{table}[]
\begin{tabular}{llllll}
    & (a,b)=(1.1, 2.0) & (a,b)=(1.1, 3.0) & (a,b)=(1.1, 4.0) & (a,b)=(1.1, 5.0) & (a,b)=(1.1, 6.0) \\
Cvg & 90.588 (0.560)   & 90.185 (0.850)   & 90.251 (0.759)   & 90.201 (0.845)   & 90.966 (1.076)   \\
PSS & 37.091 (2.245)   & 35.225 (2.758)   & 35.554 (2.281)   & 35.391 (2.661)   & 38.541 (3.470)  
\end{tabular}
\caption{Comparison of using different (a,b)).}
\label{tab:hyper_beta}
\end{table}
}

\begin{table}[t]
\centering
\resizebox{\linewidth}{!}{
\begin{tabular}{|l|cc|cc|cc|cc|}
\hline
Attack Budget & \multicolumn{2}{c|}{$\epsilon$=4/255}    & \multicolumn{2}{c|}{$\epsilon$=8/255}    & \multicolumn{2}{c|}{$\epsilon$=12/255}  & \multicolumn{2}{c|}{$\epsilon$=16/255}  \\ \hline
Metric        & \multicolumn{1}{c|}{Cvg} & PSS & \multicolumn{1}{c|}{Cvg} & PSS & \multicolumn{1}{c|}{Cvg} & PSS & \multicolumn{1}{c|}{Cvg} & PSS \\ \hline
AT            & \multicolumn{1}{c|}{92.73(0.81)	}     &   21.91(1.59)  & \multicolumn{1}{c|}{91.35(0.85)}     &   43.20(2.11)  & \multicolumn{1}{c|}{89.68(0.89)} & 78.80(5.43) &\multicolumn{1}{c|}{90.22(1.10)	}     &  136.86(7.20)   \\ \hline
AT-EM         & \multicolumn{1}{c|}{92.93(0.67)}     &  21.64(1.89)   & \multicolumn{1}{c|}{91.09(0.79)}     &   41.42(2.52)  & \multicolumn{1}{c|}{89.48(0.75)} & 76.39(4.21) &\multicolumn{1}{c|}{89.97(0.67)}     &  132.15(4.08)   \\ \hline
AT-Beta       & \multicolumn{1}{c|}{91.85(1.06)}     &   \textbf{17.17(2.05)}  & \multicolumn{1}{c|}{90.20(0.84)}     &  35.39(2.66)   & \multicolumn{1}{c|}{89.78(1.22)} & 79.79(7.09) &\multicolumn{1}{c|}{90.12(0.87)}     &  142.17(6.89)  \\ \hline
AT-EM-Beta    & \multicolumn{1}{c|}{91.83(0.76)}     &  17.59(1.21)   & \multicolumn{1}{c|}{90.81(1.00)}     &  36.17(3.73)    & \multicolumn{1}{c|}{89.97(0.92)} & 74.38(4.51) &\multicolumn{1}{c|}{90.05(0.98)}     &  131.64(5.59)   \\ \hline
\end{tabular}
}
\caption{Comparison of using different attack budgets in PGD100 with Caltech256.}
\label{tab:diff_eps}
\end{table}

\section{ Proof of Theorem}
\label{app_sec:proof_thm1}

\begin{theorem}
\label{theorem:cost_sensitive_learning_bound_for_CP_PSS}
(Learning bound for the expected size of CP prediction sets)
Let
$$
L_\Beta(f) := \sum_{k=1}^K \sigma_k \cdot \E[ \ell(f(X), Y) | r_f(X,Y) = k ],
$$ 
where $\sigma_k \sim p_\Beta(k/(K+1); a, b)$ with $a = 1.1, b = 5$.
\begin{align*}
\E_X [ | \calC_f(X) | ]
\leq 
L_\Beta(f) .
\end{align*}
\end{theorem}

\begin{proof}
(of Theorem \ref{theorem:cost_sensitive_learning_bound_for_CP_PSS})

Before the proof, we first present two key lemmas (Lemma \ref{lemma:CP_PSS_ub_by_partial_average_rank} and Lemma \ref{lemma:partial_average_rank_ub_by_L_Beta}) below. Note that our theoretical analysis only uses the original Beta function $p_{\text{Beta}}$ instead of the up-shifted version, which does not affect the conclusion since the orignal Beta-weighting can be regarded as a regularization term. Thus, this theorem shows that the Beta weighting term controls the prediction set size, while the ERM term minimizes the generalization error. As we noted in the main paper, we use the ranking starting from 0 in our implementation while the theoretical analysis assumes the ranking starts from 1. Nevertheless, we can shift the index from 1 to 0 to get the same theoretical result.

The proof for Lemma \ref{lemma:CP_PSS_ub_by_partial_average_rank} and Lemma \ref{lemma:partial_average_rank_ub_by_L_Beta} can be found in Section \ref{section:proof_for_lemma_CP_PSS_ub_by_partial_average_rank} and Section \ref{section:proof_for_lemma_partial_average_rank_ub_by_L_Beta}, respectively.
\begin{lemma}
\label{lemma:CP_PSS_ub_by_partial_average_rank}
(Expected size of CP prediction upper bounded by partial average rank)
Let $K^* = \max\{ k \in [K] : \P_{XY} [ \sum_{l=1}^{k} f(X)_{(l)} \leq \tau_{1-\alpha} | r_f(X, Y) \geq k ] \geq 1 - \alpha \}$.
\begin{align}\label{eq:CP_PSS_ub_by_partial_average_rank}
\E_X[ | \calC_f(X) | ]
\leq 
\sum_{k=1}^{K^*} k \cdot \P[ r_f(X, Y) = k ]
\end{align}
\end{lemma}

\begin{lemma}
\label{lemma:partial_average_rank_ub_by_L_Beta}
(Partial average rank upper bounded by $L_\Beta$)
\begin{align}\label{eq:partial_average_rank_ub_by_L_Beta}
\sum_{k=1}^{K^*} k \cdot \P\big[ r_f(X, Y) = k \big]
\leq 
\sum_{k=1}^K \sigma_k \cdot \E\big[ \ell(f(X), Y) \big| r_f(X, Y) = k \big]
,
\end{align}
where $\sigma_k = 3 / 5 \cdot \gamma \cdot \xi \cdot p_\Beta(k/(K+1); a, b) $, $\gamma$ is a positive constant satisfying $\bar l_k\geq k/\gamma,\forall k\in[K^*]$ and $\xi$ is a positive constant satisfying $p_k \leq \xi \cdot ( 1 - \frac{k}{K+1} )^{b-1},\forall k\in [K^*]$.
\end{lemma}

Now we can start proving Theorem \ref{theorem:cost_sensitive_learning_bound_for_CP_PSS}.
By inequality (\ref{eq:CP_PSS_ub_by_partial_average_rank}) from Lemma \ref{lemma:CP_PSS_ub_by_partial_average_rank} and (\ref{eq:partial_average_rank_ub_by_L_Beta}) from Lemma \ref{lemma:partial_average_rank_ub_by_L_Beta}, we have
\begin{align*}
\E_X\big[ | \calC_f(X) | \big]
\stackrel{ (\ref{eq:CP_PSS_ub_by_partial_average_rank}) }{ \leq } &
\sum_{k=1}^{K^*} k \cdot \P\big[ r_f(X, Y) = k \big]
\\
\stackrel{ (\ref{eq:partial_average_rank_ub_by_L_Beta}) }{ \leq } &
\underbrace{
\sum_{k=1}^K \sigma_k \cdot \E\big[ \ell(f(X), Y) | r_f(X, Y) = k \big]
}_{ = L_\Beta(f) }
,
\end{align*}
where $\sigma_k = 3 / 5 \cdot \gamma \cdot \xi \cdot p_\Beta(k/(K+1); a, b).$
This completes the proof of Theorem \ref{theorem:cost_sensitive_learning_bound_for_CP_PSS}
\end{proof}

\subsection{ Proof for Lemma \ref{lemma:CP_PSS_ub_by_partial_average_rank} }
\label{section:proof_for_lemma_CP_PSS_ub_by_partial_average_rank}

\begin{proof}
(of Lemma \ref{lemma:CP_PSS_ub_by_partial_average_rank})
We first introduce the notations. $f(X)_{(l)}$ is the $l$th sorted predictive probability with the descending order, $V(X,y)$ is the cumulative summation of $f(X)_{(l)}$, i.e., $V(X,y)=\sum_{l=1}^{y} f(X)_{(l)}$. $r_f(X,Y)$ is the TCPR of input $(X,Y)$ when using the classifier $f$. $\tau_{1-\alpha}$ is the $1-\alpha$ quantile of the conformity score at the population level and $\alpha$ is the confidence level for conformal prediction.

Now we start with the expected size of prediction sets of CP method:
\begin{align*}
&
\E_X[ | \calC_f(X) | ]
= 
\E_X \Bigg[ \sum_{y=1}^K \indicator [ V(X, y) \leq \tau_{1-\alpha} ] \Bigg]
\\
= &
\sum_{y=1}^K \E_X \Big[ \indicator [ V(X, y) \leq \tau_{1-\alpha} ] \cdot \E_Y \big[ \indicator[ r_f(X, Y) < r_f(X, y) ] + \indicator[ r_f(X, Y) \geq r_f(X, y) ] \big] \Big]
\\
= &
\sum_{y=1}^K \E_X \Big[ \indicator [ V(X, y) \leq \tau_{1-\alpha} ] \cdot \E_Y \big[ \indicator[ r_f(X, Y) < r_f(X, y) ] \big] \Big]
\\
&
+ \sum_{y=1}^K \E_X \Big[ \indicator [ V(X, y) \leq \tau_{1-\alpha} ] \cdot \E_Y \big[ \indicator[ r_f(X, Y) \geq r_f(X, y) ] \big] \Big]
\\
= &
\underbrace{ 
\sum_{y=1}^K \E_{XY} \big[ \indicator [ V(X, y) \leq \tau_{1-\alpha} ] \cdot \indicator[ r_f(X, Y) < r_f(X, y) ] \big]
}_{ = A }
\\
&
+ 
\underbrace{
\sum_{y=1}^K \E_{XY} \big[ \indicator [ V(X, y) \leq \tau_{1-\alpha} ] \cdot \indicator[ r_f(X, Y) \geq r_f(X, y) ] \big]
}_{ = B }
.
\end{align*}

Below we upper bound the two terms $A$ and $B$, respectively.
For $A$, we have
\begin{align}\label{eq:ub_A}
&
A
=
\sum_{y=1}^K \P_{XY} [ V(X, y) \leq \tau_{1-\alpha}, r_f(X, Y) < r_f(X, y) ]
\nonumber\\
= &
\sum_{y=1}^K \P_{XY} [ r_f(X, Y) < r_f(X, y) ] \cdot \P_{XY} [ V(X, y) \leq \tau_{1-\alpha} | r_f(X, Y) < r_f(X, y) ]
\nonumber\\
\stackrel{ (a) }{ \leq } &
\sum_{y=1}^K \P_{XY} [ r_f(X, Y) < r_f(X, y) ] \cdot \P_{XY} [ V(X, Y) \leq \tau_{1-\alpha} | r_f(X, Y) \leq r_f(X, y) ]
\nonumber\\
&
+ \sum_{y=1}^K \P_{XY} [ r_f(X, Y) \geq r_f(X, y) ] \cdot \bigg( 1 - \alpha + \frac{1}{n} \bigg)
- \sum_{y=1}^K \P_{XY} [ r_f(X, Y) \geq r_f(X, y) ] \cdot \bigg( 1 - \alpha + \frac{1}{n} \bigg)
\nonumber\\
\stackrel{ (b) }{ \leq } &
\sum_{y=1}^K \P_{XY} \Big( \big[ r_f(X, Y) < r_f(X, y) \big] + \big[ r_f(X, Y) \geq r_f(X, y) \big] \Big) \cdot ( 1 - \alpha )
\nonumber\\
&
- \sum_{y=1}^K \P_{XY} \big[ r_f(X, Y) \geq r_f(X, y) \big] \cdot ( 1 - \alpha )
\nonumber\\
= &
K ( 1 - \alpha ) - \sum_{y=1}^K \P_{XY} \big[ r_f(X, Y) \geq r_f(X, y) \big] \cdot ( 1 - \alpha )
,
\end{align}
where the above inequality $(a)$ is due to $\P[ V(X, y) \leq \tau_{1-\alpha} | r_f(X, Y) < r_f(X, y) ] \leq \P[ V(X, Y) \leq \tau_{1-\alpha} | r_f(X, Y) < r_f(X, y) ]$, and
the inequality $(b)$ is due to $\P[ V(X, Y) \leq \tau_{1-\alpha} | r_f(X, Y) < r_f(X, y) ] \leq \P[ V(X, Y) \leq \tau_{1-\alpha} ] \leq 1 - \alpha + \frac{1}{n}$, the latter is due to the upper bound in \cite{romano2020classification}.
It is also worth highlighting that the second term in the last line above can be re-written as the average rank:
\begin{align*}
\sum_{y=1}^K \P[ r_f(X, Y) \geq r_f(X, y) ]
=
\sum_{k=1}^K k \cdot \P[ r_f(X,Y) = k ]
.
\end{align*}

\begin{figure}[t]
     \centering
     \includegraphics[width=0.5\textwidth]{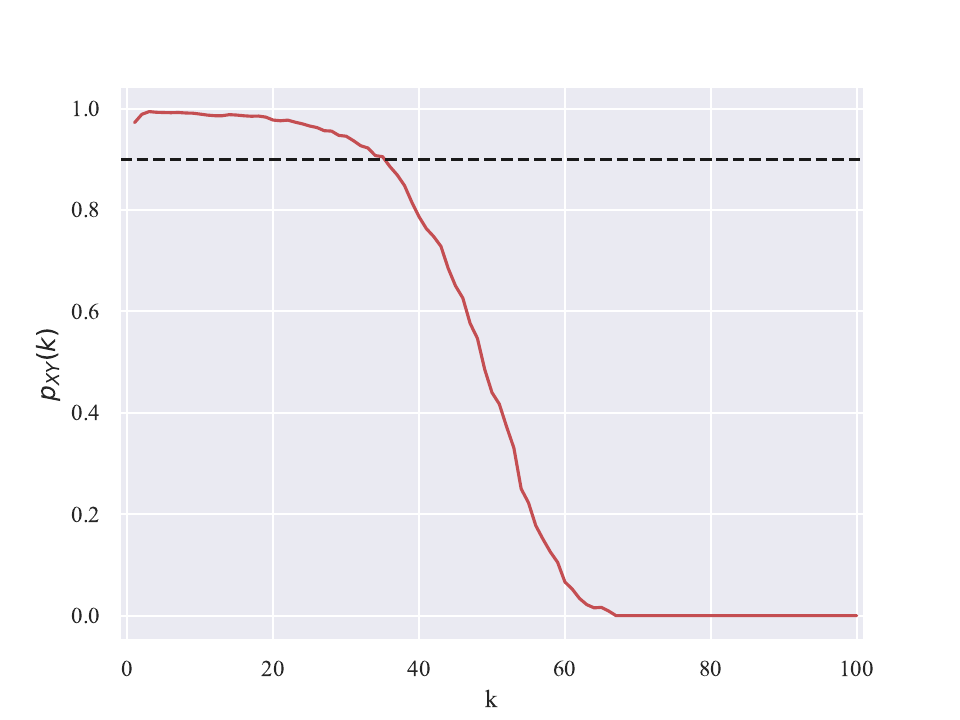}
    \caption{The empirical estimation of $\P_{XY} [ \sum_{l=1}^{k} f(X)_{(l)} \leq \tau_{1-\alpha} | r_f(X, Y) \geq k ]$ with the test set of CIFAR100 using an AT-trained model. The black dashed line is the confidence level $1-\alpha=0.9$.}
    \label{fig:p_xy}
\end{figure}

Now we turn to term $B$ and upper bound it as follows:
\begin{align}\label{eq:ub_B}
&
B - \sum_{y=1}^K \P_{XY} \big[ r_f(X, Y) \geq r_f(X, y) \big] \cdot ( 1 - \alpha )
\nonumber\\
= &
\sum_{y=1}^K \E_{XY} \Big[ \indicator \big[ V(X, y) \leq \tau_{1-\alpha} \big] \cdot \indicator[ r_f(X, Y) \geq r_f(X, y) ] \Big]
- \sum_{y=1}^K \P_{XY} \big[ r_f(X, Y) \geq r_f(X, y) \big] \cdot ( 1 - \alpha )
\nonumber\\
= &
\sum_{y=1}^K \P_{XY} \big[ V(X, y) \leq \tau_{1-\alpha} ,~ r_f(X, Y) \geq r_f(X, y) \big]
- \sum_{y=1}^K \P_{XY} \big[ r_f(X, Y) \geq r_f(X, y) \big] \cdot ( 1 - \alpha )
\nonumber\\
= &
\sum_{y=1}^K \P_{XY} \big[ r_f(X, Y) \geq r_f(X, y) \big] \cdot \P_{XY} \big[ V(X, y) \leq \tau_{1-\alpha} \big| r_f(X, Y) \geq r_f(X, y) \big]
\nonumber\\
&
- \sum_{y=1}^K \P_{XY} \big[ r_f(X, Y) \geq r_f(X, y) \big] \cdot ( 1 - \alpha )
\nonumber\\
= &
\sum_{y=1}^K \P_{XY} \big[ r_f(X, Y) \geq r_f(X, y) \big] \cdot \Big( \P_{XY} \big[ V(X, y) \leq \tau_{1-\alpha} \big| r_f(X, Y) \geq r_f(X, y) \big] - ( 1 - \alpha ) \Big)
\nonumber\\
= &
\sum_{y=1}^K \P_{XY} \big[ r_f(X, Y) \geq r_f(X, y) \big] \cdot \Big( \P_{XY} \big[ \sum_{l=1}^{r_f(X,y)} f(X)_{(l)} \leq \tau_{1-\alpha} \big| r_f(X, Y) \geq r_f(X, y) \big] - ( 1 - \alpha ) \Big)
\nonumber\\
\stackrel{ (a) }{ = } &
\sum_{k=1}^K \P_{XY} \big[ r_f(X, Y) \geq k \big] \cdot \Big( 
\underbrace{ 
\P_{XY} \big[ \sum_{l=1}^{k} f(X)_{(l)} \leq \tau_{1-\alpha} \big| r_f(X, Y) \geq k \big]
}_{ =: H(k) }
- ( 1 - \alpha ) \Big)
\nonumber\\
= &
\sum_{k=1}^{K^*} \P_{XY} \big[ r_f(X, Y) \geq k \big] \cdot \Big( \P_{XY} \big[ \sum_{l=1}^{k} f(X)_{(l)} \leq \tau_{1-\alpha} \big| r_f(X, Y) \geq k \big] - ( 1 - \alpha ) \Big)
\nonumber\\
&
+ \sum_{k=K^*+1}^K \P_{XY} \big[ r_f(X, Y) \geq k \big] \cdot \Big( \P_{XY} \big[ \sum_{l=1}^{k} f(X)_{(l)} \leq \tau_{1-\alpha} \big| r_f(X, Y) \geq k \big] - ( 1 - \alpha ) \Big)
\nonumber\\
\stackrel{ (b) }{ \leq} &
\sum_{k=1}^{K^*} \P_{XY} \big[ r_f(X, Y) \geq k \big] \cdot \Big( \P_{XY} \big[ \sum_{l=1}^{k} f(X)_{(l)} \leq \tau_{1-\alpha} \big| r_f(X, Y) \geq k \big] - ( 1 - \alpha ) \Big)
,
\end{align}
where the above equality $(a)$ is due to $k = r_f(X, y)$,
the inequality $(b)$ is due to the definition of $K^* = \max\{ k \in [K] : \P_{XY} [ \sum_{l=1}^{k} f(X)_{(l)} \leq \tau_{1-\alpha} | r_f(X, Y) \geq k ] \geq 1 - \alpha ) \}$ and the assumption of the monotonically decreasing function of $H(k) = \P_{XY} [ \sum_{l=1}^{k} f(X)_{(l)} \leq \tau_{1-\alpha} | r_f(X, Y) \geq k ]$ in $k$.

We plot the empirical estimation of $H(k) = \P_{XY} [ \sum_{l=1}^{k} f(X)_{(l)} \leq \tau_{1-\alpha} | r_f(X, Y) \geq k ]$ with the test set (adversarially attacked by PGD100) of CIFAR100 using an AT-trained model in Fig.~\ref{fig:p_xy}, which validates our assumption on the monotonically decreasing property of this function.

Combining the above two inequalities (\ref{eq:ub_A}) and (\ref{eq:ub_B}) together, we have
\begin{align*}
&
\E_X\big[ | \calC_f(X) | \big]
\\
\leq &
K ( 1 - \alpha )
+ \sum_{k=1}^{K^*} \P_{XY} \big[ r_f(X, Y) \geq k \big] \cdot \Big( \P_{XY} \big[ \sum_{l=1}^{k} f(X)_{(l)} \leq \tau_{1-\alpha} \big| r_f(X, Y) \geq k \big] - ( 1 - \alpha ) \Big)
\\
\leq &
K ( 1 - \alpha )
+ \sum_{k=1}^{K^*} \P_{XY} \big[ r_f(X, Y) \geq k \big] 
\\
= &
K ( 1 - \alpha )
+ \sum_{k=1}^{K^*} k \cdot \P_{XY} \big[ r_f(X, Y) = k \big]
.
\end{align*}
After dropping the constant (since the training does not optimize the constant), this completes the proof for Lemma \ref{lemma:CP_PSS_ub_by_partial_average_rank}.
\end{proof}

\subsection{ Proof for Lemma \ref{lemma:partial_average_rank_ub_by_L_Beta} }
\label{section:proof_for_lemma_partial_average_rank_ub_by_L_Beta}

\begin{proof}
(of Lemma \ref{lemma:partial_average_rank_ub_by_L_Beta})

The proof for Lemma \ref{lemma:partial_average_rank_ub_by_L_Beta} needs the following technical lemma, which is proved in Section \ref{section:proof_for_lemma_partial_average_rank_ub_by_L_Beta}.
\begin{lemma}
\label{lemma:ub_of_Gamma_functions}
(Upper bound of Gamma functions)
Let $a = 1.1, b = 5$.
Then we have the following inequality
\begin{align*}
\frac{ \Gamma(a+b) }{ \Gamma(a) \Gamma(b) }
\leq 
3 / 10.
\end{align*}
\end{lemma}

We use $p_k$ to denote $\P\big[ r_f(X, Y) = k \big]$ and $\bar l_k$ to denote $\E\big[ \ell(f(X), Y) \big| r_f(X, Y) = k \big]$.
\begin{align*}
&
\frac{ k \cdot p_k }{ \bar \ell_k }
\stackrel{ (a) } { \leq }
\frac{ k \cdot p_k }{ k / \gamma }
\stackrel{ (b) } { \leq }
\gamma \cdot \xi \cdot (1 - \frac{k}{K+1})^{b-1} 
\stackrel{ (c) } { \leq }
\gamma \cdot \xi \cdot \frac{ 2 }{(K+1)^{a-1}} \cdot (1 - \frac{k}{K+1})^{b-1}
\\
\stackrel{ (d) } { \leq } &
2 \gamma \cdot \xi \cdot \frac{k^{a-1}}{(K+1)^{a-1}} \cdot (1 - \frac{k}{K+1})^{b-1}  \cdot 
\frac{ \Gamma(a) \Gamma(b) }{ \Gamma(a+b) } \cdot
\frac{ \Gamma(a+b) }{ \Gamma(a) \Gamma(b) }
\\
= &
2 \gamma \cdot \xi \cdot p_\Beta(k/(K+1); a, b)  \cdot 
\frac{ \Gamma(a+b) }{ \Gamma(a) \Gamma(b) }
\\
\stackrel{ (e) } { \leq } &
3 / 5 \cdot \gamma \cdot \xi \cdot p_\Beta(k/(K+1); a, b) 
=
\sigma_k
,
\end{align*}
where the above inequality $(a)$ is due to the assumption $\bar \ell_k \geq k / \gamma$,
the inequality $(b)$ is due to the assumption $p_k \leq \xi \cdot ( 1 - \frac{k}{K+1} )^{b-1}$,
the inequality $(c)$ is due to the assumption $K \leq 2^{10}$ and $a-1 = 1/10$,
the inequality $(d)$ is due to $1 \leq k$,
the inequality $(e)$ is due to Lemma \ref{lemma:ub_of_Gamma_functions}. 

We plot the curve of $\bar l_k$ versus $k/\gamma$ ($\gamma=10$) and $p_k$ versus $\xi(1-\frac{k}{K+1})^{b-1}$ ($\xi=0.5$, $b=5$) with an adversarially trained model on CIFAR100 in Fig.~\ref{fig:assumption_simulation}, indicating that the two assumptions are valid in practice. This proves the inequality 
$$\sum_{k=1}^{K^*} k \cdot \P[ r_f(X, Y) = k ]
\leq 
\sum_{k=1}^{K^*} \sigma_k \cdot \E\big[ \ell(f(X), Y) \big| r_f(X, Y) = k \big].$$

\begin{figure}[t]
\centering
\includegraphics[width=0.9\textwidth]{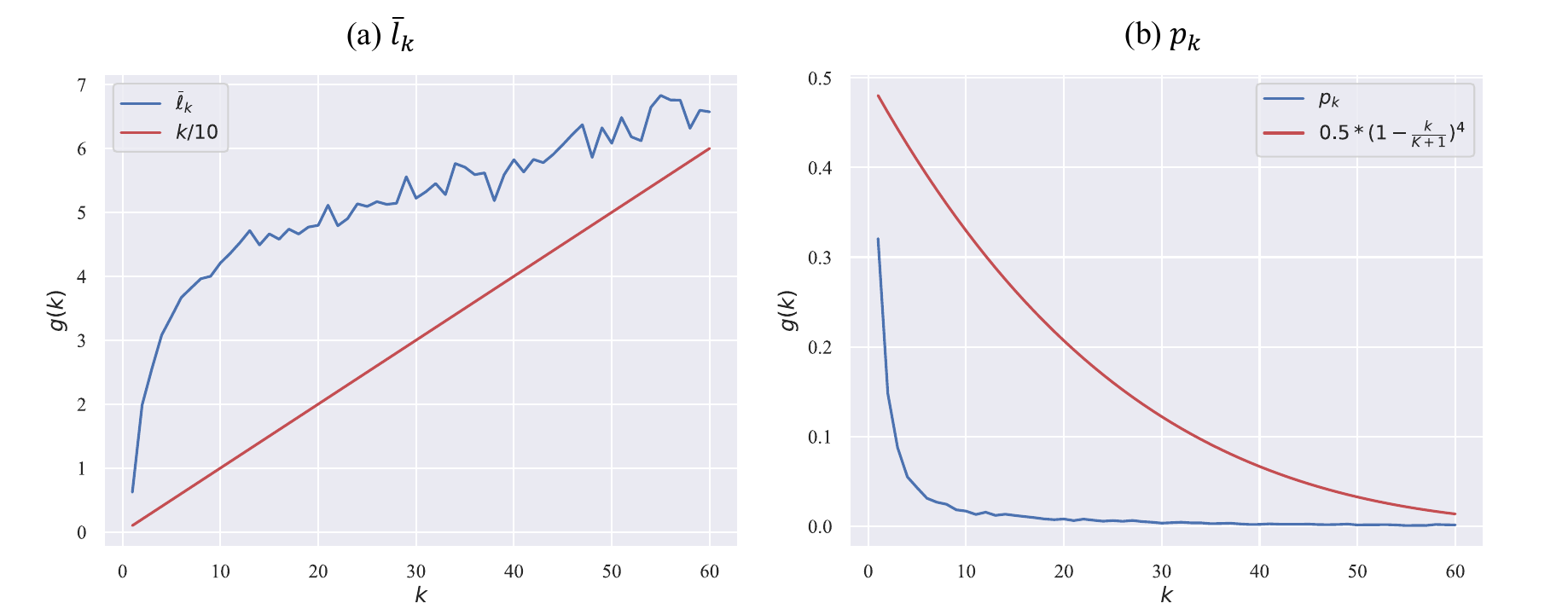}
\caption{The empirical estimation of \textbf{(a)} $\bar l_k$ and \textbf{(b)} $p_k$ on the test set of CIFAR100, with an adversarially trained model on CIFAR100.
}
\label{fig:assumption_simulation}
\end{figure}

This completes the proof of Lemma \ref{lemma:partial_average_rank_ub_by_L_Beta}.
\end{proof}

\subsection{ Proof for Lemma \ref{lemma:ub_of_Gamma_functions} }
\label{section:proof_lemma_ub_of_Gamma_functions}

\begin{proof}
(of Lemma \ref{lemma:ub_of_Gamma_functions})

We use Weierstrass's definition for Gamma function:
\begin{align*}
\Gamma(z)
=
\frac{ \exp( -\gamma z ) }{ z } \prod_{i=1}^\infty ( 1 + z/i )^{-1} \cdot \exp( z / i )
,
\end{align*}
where $\gamma_0$ is the Euler–Mascheroni constant.
Now we start the proof:
\begin{align*}
\frac{ \Gamma(a+b) }{ \Gamma(a) \Gamma(b) }
= &
\frac{ \exp(-\gamma_0 a) }{ a } \cdot
\frac{\exp(-\gamma_0 b) }{ b } \cdot
\frac{ a+b }{\exp(-\gamma_0 (a+b)) } \cdot
\\
&
\prod_{i=1}^\infty (1+a/i)^{-1} \cdot (1+b/i)^{-1} \cdot (1+(a+b)/i) \cdot \exp( a/i + b/i - (a+b)/i )
\\
= &
\frac{ a+b }{ ab } \cdot  
\prod_{i=1}^\infty \frac{ 1 + (a+b)/i }{ ( 1+a/i )( 1+b/i ) }
\stackrel{ (a) }{ \leq }
\frac{ 6 }{ 5 } \cdot \prod_{i=1}^\infty \frac{ ( i + a + b ) i }{ (i+a) (i+b) }
\\
= &
\frac{ 6 }{ 5 } \cdot \prod_{i=1}^\infty \Big( 1 - \frac{ ab }{ i^2 + (a+b)i + ab } \Big)
\stackrel{ (b) }{ \leq }
\frac{ 6 }{ 5 } \cdot \prod_{i=1}^\infty \exp\Big( - \frac{ ab }{ i^2 + (a+b)i + ab } \Big)
\\
= &
\frac{ 6 }{ 5 } \cdot \exp \Big( -ab \cdot \sum_{i=1}^\infty \frac{ 1 }{ i^2 + (a+b)i + ab } \Big)
\stackrel{ (c) }{ = }
\frac{ 6 }{ 5 } \cdot \exp\Big( -ab \cdot \sum_{i=1}^\infty \frac{ 1 }{ i^2 + 61/10 \cdot i + 11/2 } \Big)
\\
\stackrel{ (d) }{ < } &
\frac{ 6 }{ 5 } \cdot \exp( - 3 ab / 10 )
\stackrel{ (e) }{ < }
\frac{ 6 }{ 5 } \cdot \exp( - 3 / 2 )
<
3/10
,
\end{align*}
where the inequality $(a)$ is due to $1/a < 1, 1/b = 1/5$,
the inequality $(b)$ is due to $1+x \leq \exp(x)$,
the equality $(c)$ is due to $a = 1.1, b = 5$,
the inequality $(d)$ is due to 
$$
\sum_{i=1}^\infty \frac{ 1 }{ i^2 + 61/10 \cdot i + 11/2 }
>
\sum_{i=1}^{100} \frac{ 1 }{ i^2 + 61/10 \cdot i + 11/2 }
>
3/10,$$
and the inequality $(e)$ is due to $a > 1, b = 5$.
This completes the proof of Lemma \ref{lemma:ub_of_Gamma_functions}.
\end{proof}


\end{document}